\documentclass[table]{gtech}

\usepackage{amssymb}
\usepackage{bigdelim}
\usepackage{longtable}
\usepackage{tabularray}
\usepackage{wrapfig}
\usepackage{float}
\usepackage{datatool}
\usepackage{enumitem}
\ifPDFTeX
  \usepackage{tgpagella}
  \usepackage{mathpazo}
\else
  \usepackage{unicode-math}
\fi
\usepackage{inconsolata}
\usepackage{makecell}
\usepackage{adjustbox}
\usepackage{tablefootnote}
\usepackage{array}
\usepackage{nicefrac}
\usepackage{xspace}
\usepackage{amsthm}
\usepackage{amsmath}
\usepackage{arydshln}
\usepackage{pifont}
\usepackage{tabulary}
\usepackage{fontawesome5}
\usepackage{bbding}
\usepackage{multicol}
\usepackage{mathrsfs}
\usepackage{algorithm}
\usepackage{algpseudocode}

\theoremstyle{plain}
\newtheorem{theorem}{Theorem}[section]

\newtheorem{corollary}[theorem]{Corollary}

\definecolor{rawone}{HTML}{F07171}
\definecolor{rawtwo}{HTML}{74B3E8}
\definecolor{mergeone}{HTML}{B71C1C}
\definecolor{mergetwo}{HTML}{0D47A1}
\definecolor{interp}{HTML}{43A047}

\definecolor{tcxcol}{HTML}{1565C0} 
\definecolor{hzhcol}{HTML}{C62828} 
\definecolor{hxlcol}{HTML}{6A1B9A} 

\tcbset{
  collaboratorbox/.style={
    enhanced,
    breakable,
    boxrule=0.6pt,
    arc=2pt,
    left=6pt,
    right=6pt,
    top=4pt,
    bottom=4pt,
    before skip=6pt,
    after skip=6pt,
    fonttitle=\bfseries\sffamily\small,
    coltitle=white,
  }
}
\newtcolorbox{tcxbox}[1][]{
  collaboratorbox,
  colback=tcxcol!8,
  colframe=tcxcol,
  title={TCX\ifstrempty{#1}{}{: #1}}
}
\newtcolorbox{hzhbox}[1][]{
  collaboratorbox,
  colback=hzhcol!8,
  colframe=hzhcol,
  title={HZH\ifstrempty{#1}{}{: #1}}
}
\newtcolorbox{hxlbox}[1][]{
  collaboratorbox,
  colback=hxlcol!8,
  colframe=hxlcol,
  title={HXL\ifstrempty{#1}{}{: #1}}
}

\title{\scalebox{0.97}{\shortstack[l]{
Trajectory Soup: Pushing the Compute-Scaling \\Frontier of LLM Mid-training via Diverse Trajectories}}}

\author[1,2 * \dag]{\vspace{0.5em}Zhehao Huang}
\author[1 *]{Changxin Tian}
\author[1]{Qingyuan Yang}
\author[1]{Kunlong Chen}
\author[1]{Ziqi Liu}
\author[1 \ddag]{\\Zhiqiang Zhang}
\author[2 \ddag]{Xiaolin Huang}
\author[1 ]{Jun Zhou}

\affiliation[1]{Ling Team, Ant Group}
\affiliation[2]{Shanghai Jiao Tong University}
\contribution[*]{Equal contribution}
\contribution[\dag]{Contribution during internship at Ant Group}
\contribution[\ddag]{Corresponding author}

\abstract{Mid-training equips pretrained large language models with specialized and reasoning capabilities, but the returns of this stage are bounded since additional serial compute yields little further downstream improvement and can even degrade some capabilities, which places a practical ceiling on how much compute mid-training absorbs. We revisit how this compute should be allocated to a single run or multiple similar optimizations. We find that branches forked from a shared checkpoint under various controlled recipe reaches measurably different regions of parameter space, and establish a form of compatible diversity that extending one run cannot supply. Therefore, we introduce \textbf{Trajectory Soup}, which distributes a mid-training budget over several independent branches, and consolidates strongest checkpoints selected on validation through intra- and inter-trajectory averaging into a single model. A local bias and variance analysis separates the two averaging levels, showing that inter-trajectory averaging removes residual error beyond the reach of averaging within a trajectory, while checkpoint selection carries a bias that bounds how many checkpoints are worth merging. Across model scales, learning-rate schedules, token budgets, and trajectory counts, Trajectory Soup improves aggregate downstream performance over the strongest single-trajectory average under matched budgets and keeps improving as budgets expand, with the advantage preserved after an identical post-training pipeline. These results position trajectory allocation and merging as a practical way to extend the compute-scaling frontier of mid-training beyond serial saturation.}
\date{\today}
\gtechdata[Correspondence]{%
  \begin{tabular}[t]{@{}l@{}}
    \email{\{tianchangxin.tcx,lingyao.zzq\}@antgroup.com}\\
    \email{\{kinght\_h,xiaolinhuang\}@sjtu.edu.cn}
  \end{tabular}%
}

\begin{document}
\maketitle

\begin{figure*}[tb]
    \centering
    \captionsetup{skip=5pt}
    \captionsetup[subfigure]{skip=1pt}
    \begin{subfigure}[b]{0.485\textwidth}
        \centering
        \includegraphics[width=\linewidth]{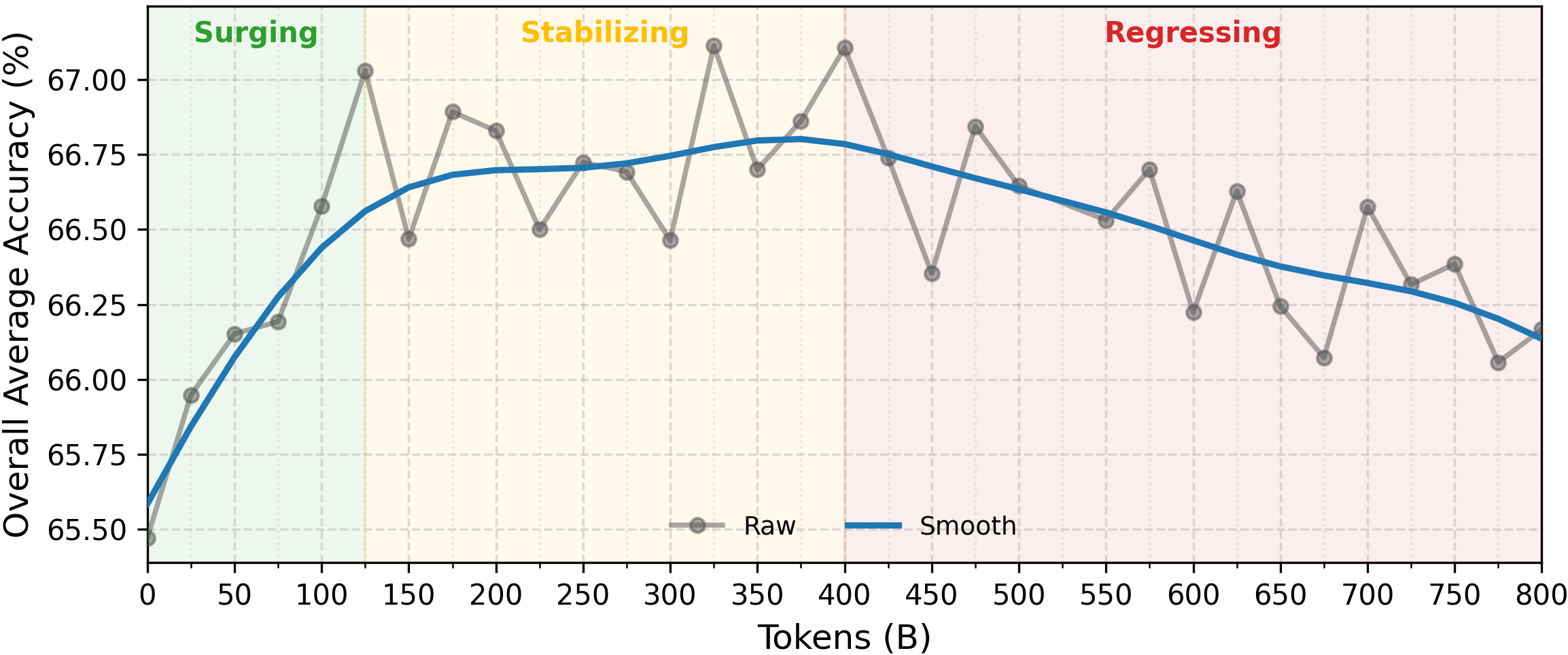}
        \caption{}
        \label{fig:midtraining-saturation}
    \end{subfigure}
    \hfill
    \begin{subfigure}[b]{0.485\textwidth}
        \centering
        \includegraphics[width=\linewidth]{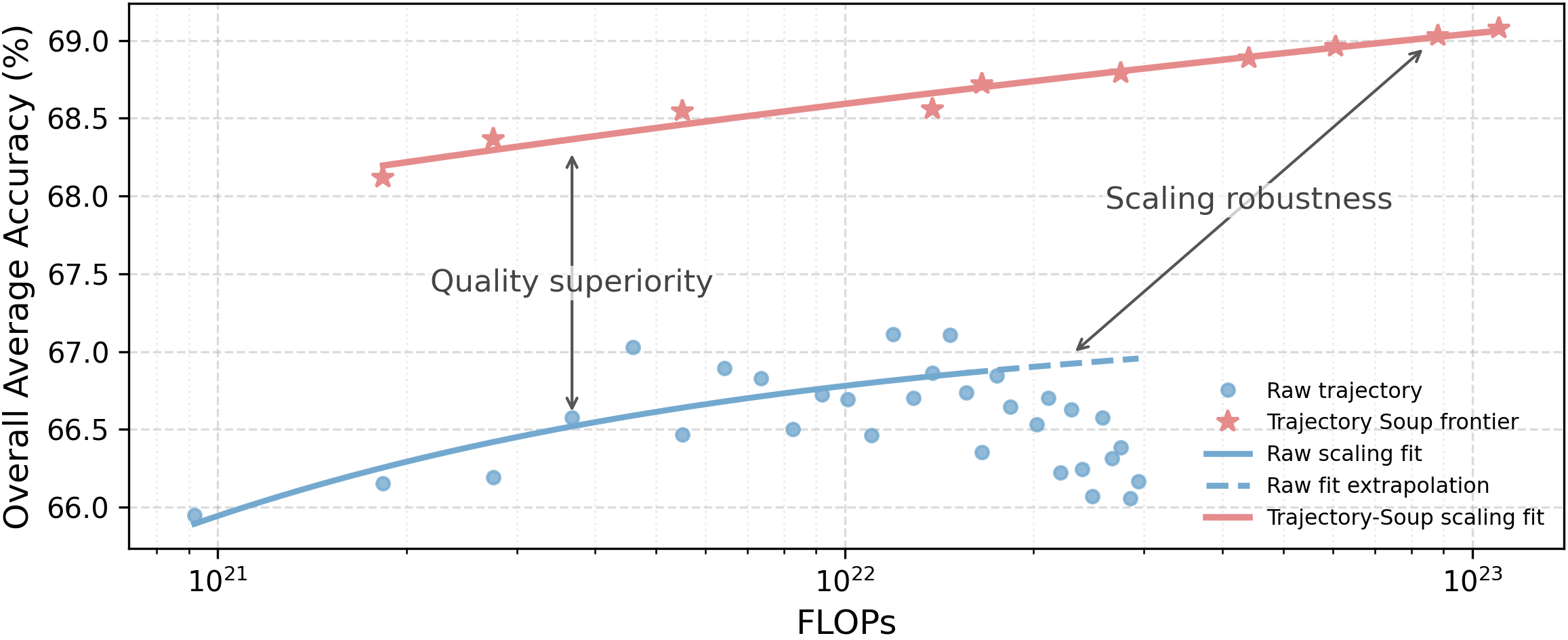}
        \caption{}
        \label{fig:trajectory-soup-compute-allocation}
    \end{subfigure}
    \caption{Motivation and compute-allocation overview of Trajectory Soup. (a) Performance saturation along serial mid-training trajectories. (b) Compute scaling of raw training and Trajectory Soup. Overall Average Accuracy is plotted against cumulative training FLOPs on a logarithmic scale.}
    \label{fig:intro-motivation}
\end{figure*}

\section{Introduction}
\label{sec:introduction}

Mid-training plays an important role in developing specialized and reasoning capabilities in pretrained large language models (LLMs)~\citep{liu2025midtraining,shao2024deepseekmath}. However, unlike in pretraining, increasing mid-training compute does not always lead to better downstream performance. As illustrated in \textbf{Fig}.~\ref{fig:midtraining-saturation}, downstream performance improves rapidly early in training and then saturates, further training can even degrade capabilities. These observations expose a practical \emph{scaling wall} in  mid-training: beyond a certain point, extending the single training run brings little downstream improvement. This motivates us to explore alternative ways of allocating mid-training compute to sustain downstream gains. 

To this end, we consider distributing compute across multiple shorter trajectories. Starting from a common pretrained checkpoint, independent branches can follow different optimization paths through controlled variation of the training recipe~\citep{frankle2020linear}, potentially producing complementary updates even on the same data distribution. These potentially complementary updates can then be consolidated into a single model through model merging. However, existing methods typically focus on either intra-trajectory merging, which averages checkpoints within a single run~\citep{izmailov2018averaging}, or inter-trajectory merging, which combines runs through endpoint averaging~\citep{wortsman2022model} or periodic aggregation~\citep{douillard2023diloco}. Whether combining these two forms of merging can make compute allocation across trajectories more effective remains underexplored. Inspired by the compute-allocation perspective of scaling laws~\citep{kaplan2020scaling,hoffmann2022training}, we study how to divide a mid-training budget between trajectory count and trajectory length. This leads us to ask: \textit{whether exploiting both at once turns a fixed compute budget spread over several trajectories into a single model that outperforms a longer serial run?}

Our empirical observations and theoretical analysis support this possibility. First, after within-trajectory averaging, the examined branches retain distinct dominant directions. Interpolating between their averaged anchors further reduces validation loss, indicating that useful differences remain after temporal averaging. Second, a local bias--variance analysis separates within-trajectory fluctuations from persistent branch variation and characterizes how checkpoint selection trades variance reduction against mean displacement. Building on these findings, we propose \textbf{Trajectory Soup}, which combines both stages. It forks several independent branches from a shared pretrained checkpoint and trains each for the same number of tokens. Within each branch, it ranks checkpoints by validation set and averages a fixed number of the lowest-loss checkpoints into a trajectory anchor. It then averages these anchors into a single model, retaining the original architecture, tokenizer, and inference procedure. We evaluate two budget settings: budget reallocation, where branches share a fixed total token budget, and budget expansion, where we add branches while keeping the per-branch token budget fixed.

Trajectory Soup turns compute into downstream gains where serial training no longer does. Under a matched token budget, Trajectory Soup surpasses the strongest intra-trajectory average as well as inter-trajectory baselines that keep only trajectory endpoints or that average the entire pooled set of checkpoints. Expanding the budget with additional trajectories widens the margin instead of exhausting it. \textbf{Fig}.~\ref{fig:trajectory-soup-compute-allocation} shows that Trajectory Soup frontier lies above the raw scaling curve at every matched compute level, and it keeps rising at a steady rate while the raw fit bends over and its checkpoints scatter downward as the budget grows. The frontier also improves as trajectories are added, and each trajectory needs to contribute only its strongest few checkpoints, so broad coverage of the loss basin rather than deep sampling within it carries the gain. The advantage persists after supervised fine-tuning, so scaling trajectories pays off end to end rather than only at the mid-training checkpoint. In summary, our contributions are threefold.

\begin{itemize}[leftmargin=*]
    \item We recast compute allocation for LLM mid-training as a question of how many optimization trajectories to run rather than how long to run a single one, and identify trajectory count as an allocation axis that remains productive after the length of a single trajectory has saturated. Trajectories separated only by controlled variation of the recipe stay geometrically distinct, and interpolating between them improves on both endpoints.

    \item We introduce Trajectory Soup, which selects the strongest checkpoints inside each trajectory and averages the resulting anchors across trajectories, yielding a single model with unchanged inference cost. A local bias and variance analysis grounds this design by separating what the two averaging levels contribute, quantifying the selection-dependent bias, showing that uniform weights minimize the variance term, and predicting a finite preferred checkpoint count that decreases as trajectories are added.

    \item We validate the method across model scales, learning-rate schedules, training budgets, and trajectory counts, show that the advantage is retained after a shared post-training pipeline, and isolate trajectory diversity from checkpoint sampling density as the source of the gains.
\end{itemize}

\section{Empirical Motivation: Saturation and Complementary Trajectories}
\label{sec:motivation}

\textbf{Serial mid-training exhibits diminishing downstream returns.}
Our scaling baseline fixes the model architecture and a high-quality mid-training corpus and adopts the best configuration found in hyperparameter search. As \textbf{Fig}.~\ref{fig:midtraining-saturation} shows, downstream performance first improves with the token budget, then plateaus, and eventually declines. The serial-training recipe therefore reaches a practical saturation regime beyond which further tokens stop paying off. Two mechanisms plausibly contribute to this decline. With model capacity fixed, prolonged exposure to a limited corpus offers diminishing returns, consistent with findings on the decreasing value of repeated training tokens~\citep{muennighoff2023scaling}. Separately, stochastic fluctuations along the optimization path can leave a checkpoint displaced from nearby parameter regions of lower validation loss~\citep{mandt2017stochastic}. The second mechanism can be mitigated by weight averaging~\citep{zhou2026extra}, as the raw and averaged checkpoints of \textbf{Fig}.~\ref{fig:interpolation} indicate. However, once local fluctuations are largely attenuated, additional checkpoints contribute less noise reduction. Averaging within a single trajectory therefore runs into diminishing returns. 

The limitation of extending the same run motivates us to allocate additional compute to several branches instead. The remainder of this section tests two prerequisites for that allocation: (1) whether controlled recipe perturbations produce distinct optimization directions, and (2) whether combining the resulting branch anchors reaches lower validation loss.

\begin{figure}[t]
    \centering
    \begin{minipage}[t]{0.36\textwidth}
        \centering
        \includegraphics[width=\linewidth]{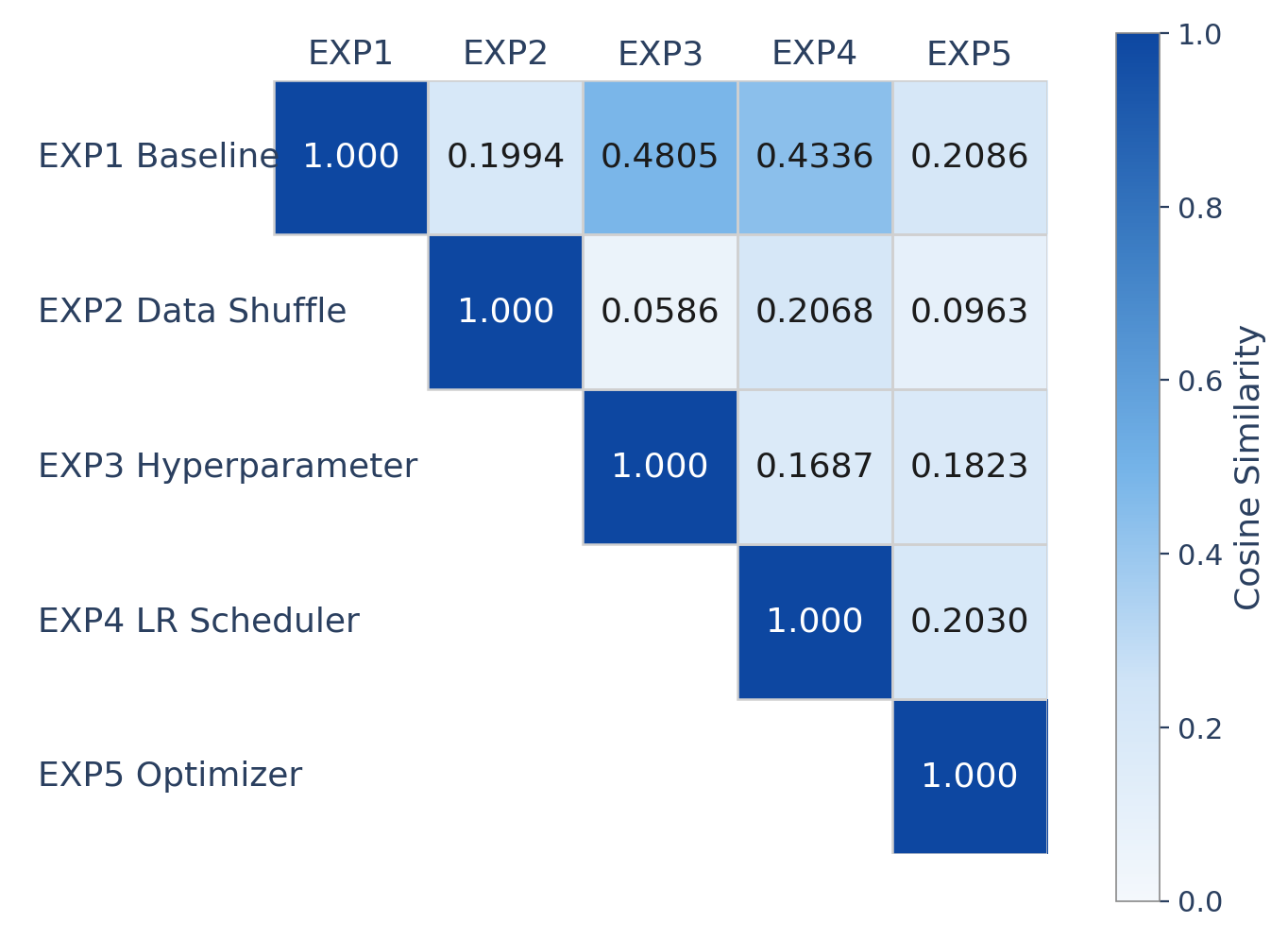}
        \caption{\textbf{Compatible branch directions overlap.} Pairwise cosine similarity between leading directions for five compatible branches. All pairs only have weak overlap and partial optimization similarity.}
        \label{fig:direction-heatmap}
    \end{minipage}\hfill
    \begin{minipage}[t]{0.62\textwidth}
        \centering
        \includegraphics[width=\linewidth]{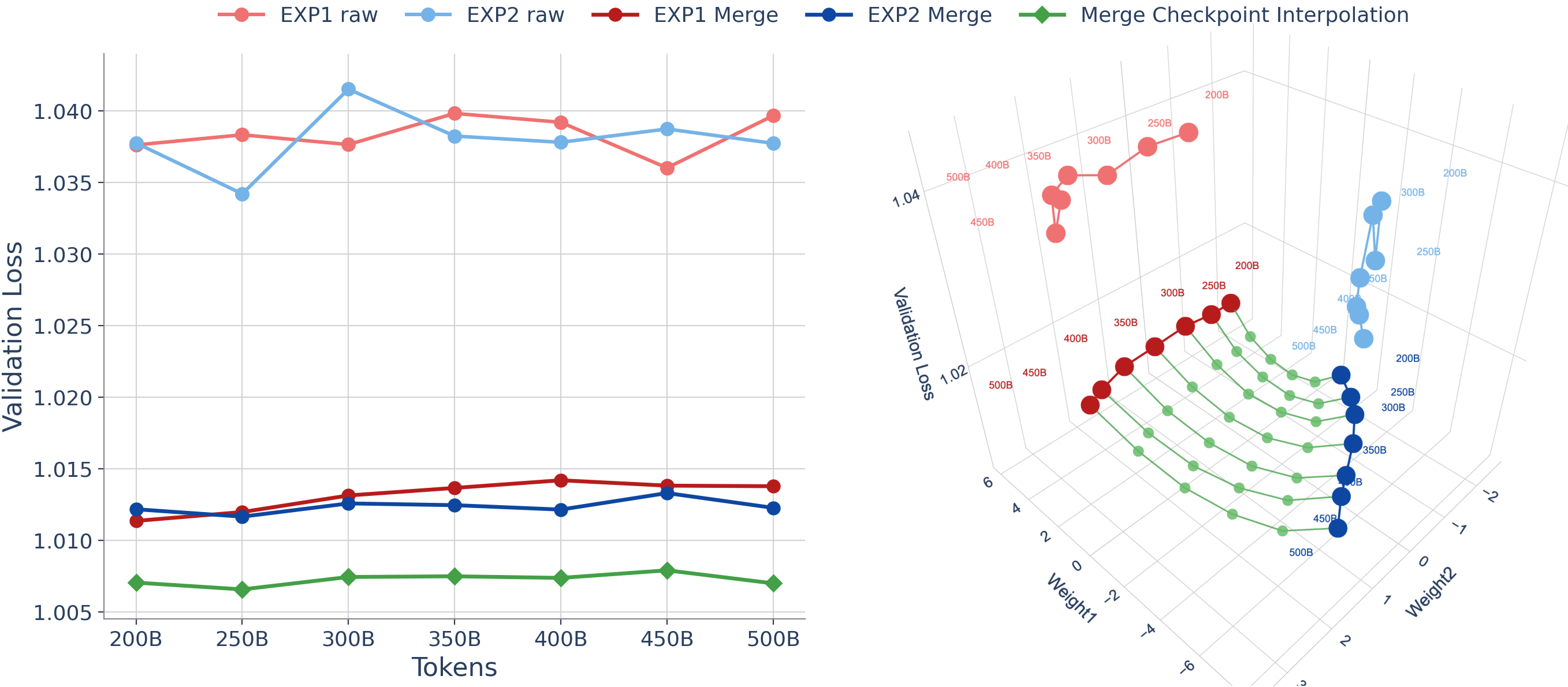}
        \caption{\textbf{Validation loss within and across trajectories.} Left: validation loss against consumed tokens for the raw checkpoints of EXP1 and EXP2, for their within-trajectory merges, and for the best interpolation between the two merged anchors. Right: the same models in a two-dimensional projection of the weight space, with one interpolation path per token budget.}
        \label{fig:interpolation}
    \end{minipage}
    \vspace{-10pt}
\end{figure}

\textbf{Intra-trajectory averaging reveals partially overlapping branch directions.}
We first examine the directional diversity across branches that are comparably optimized. We parallelly run five trajectories under various recipes (denoted as $\mathrm{EXP}1\sim\mathrm{EXP}5$). Each branch forks from the baseline run and perturbed by a controlled recipe change, such as changing the data order, learning rate, batch size, learning-rate schedule, or optimizer. All branches satisfy the training-loss screening criterion in \textbf{Sec}.~\ref{sec:problem-formulation}, controlling differences in measured optimization quality. Now we try to discover whether these branches explore distinct optimization directions. \citet{zhou2026extra} report that averaging can reveal an approximately rank-one structure within a late-stage training trajectory. Following this observation, we smooth each branch through intra-trajectory averaging, extract its leading principal component, and compute pairwise cosine similarities with other trajectories. \textbf{Fig}.~\ref{fig:direction-heatmap} shows positive but limited alignment among these branches, indicating that controlled recipe perturbations produce distinct dominant directions with partial overlap. Directional diversity alone, however, does not establish a merging benefit, which depends on the validation-loss geometry between the resulting anchors rather than on the angle between them. Detailed protocol of trajectory geometry construction is provided in Appendix~\ref{app:geometry}.

\textbf{Inter-trajectory averaging reaches lower validation loss.}
We assess the benefit of combining intra-trajectory averaged anchors through linear interpolation~\citep{frankle2020linear}. At each matched token horizon, we take the averaged checkpoints of two screened branches and sweep the interpolation coefficient between them, following the protocol in Appendix~\ref{app:geometry}. \textbf{Fig}.~\ref{fig:interpolation} shows inter-trajectory interpolation achieves lower validation loss than intra-trajectory averaging alone. The right panel illustrates the corresponding geometry: branches departing from a shared pretrained checkpoint follow distinct optimization directions, while interpolation between their averaged anchors reaches a lower-loss region. This coexistence of directional diversity and compatibility under parameter averaging is consistent with prior findings on linear mode connectivity and model soups~\citep{frankle2020linear,wortsman2022model}. Together, the direction and interpolation analyses identify distinct optimization paths whose progress can be combined beneficially. These findings motivate allocating additional compute to multiple compatible branches as serial training approaches saturation. \textbf{Sec}.~\ref{sec:method} formalizes this construction as Trajectory Soup.

\section{Trajectory Soup}
\label{sec:method}

\subsection{Problem Formulation and Budget Accounting}
\label{sec:problem-formulation}

Let $\theta_0$ be a shared pre-trained checkpoint and let $\mathcal P$ be the target training distribution, which we keep common to all branches. Branch $n\in\{1,\ldots,N\}$ follows recipe $\psi_n$ and produces parameters $\theta_n(i)$ after consuming $i$ mid-training tokens. All branches retain the same architecture, so their parameters can be combined coordinate by coordinate. For a common per-branch budget $t$, the total token budget is $T=Nt$. Each recipe $\psi_n$ is obtained by modifying the best configuration found in the baseline hyperparameter search along one or more of its data order, learning rate, batch size, learning-rate schedule, or optimizer. Because such perturbations can also destabilize a run, we admit a branch only when a compatibility observable $\mathcal J$ stays within a tolerance $\varepsilon$, that is, when $\mathcal J(\theta_n(t))\leq\varepsilon$. In our implementation $\mathcal J$ is the relative training-loss gap to the baseline at the same token horizon and $\varepsilon=0.01$. This screening keeps the branches comparable in optimization quality, while their parameter-space directions remain diverse. Changes to the data mixture are a possible extension, and we leave their effects to future work.

\subsection{Trajectory Soup}
\label{sec:trajectory-soup}

The observations in \textbf{Sec}.~\ref{sec:motivation} suggest merging at two levels: within a branch where averaging attenuates local fluctuations along one path, and across branches where averaging consolidates directions that individual paths do not share. Trajectory Soup performs these two stages in order.

\begin{figure}[tbp]
    \begin{minipage}[t]{0.49\textwidth}
    \vspace{0pt}
    \input{_figures/trajectory_soup_algorithm.tex}
    \end{minipage}\hfill
    \begin{minipage}[t]{0.48\textwidth}
    \vspace{-1pt}
    \input{_figures/trajectory_soup_combined.tex}
    \end{minipage}
\end{figure}

\textbf{Step 1: Intra-Trajectory Merging: selecting and averaging checkpoints within each branch.}
\label{sec:checkpoint-selection}
Let $\mathcal C_n(t)$ be the candidate token positions of checkpoints saved from branch $n$ by $t$. Checkpoints beyond this horizon are excluded. We rank the candidates of each branch by validation loss $\widehat{\mathcal L}_{\mathrm{val}}(\theta)$, keep the best $K$ positions, and average them into a branch anchor, so that within-trajectory variation is attenuated before any branch is combined with another,
\begin{equation}
\mathcal I_n(t,K)
\in\underset{\mathcal I\subseteq\mathcal C_n(t),|\mathcal I|=K}{\arg\min}
\sum_{i\in\mathcal I}\widehat{\mathcal L}_{\mathrm{val}}(\theta_n(i)),
\qquad
\bar\theta_n(t,K)
=\frac1K\sum_{i\in\mathcal I_n(t,K)}\theta_n(i).
\label{eq:topk-selection}
\end{equation}

\textbf{Step 2: Inter-Trajectory Merging: combining branch anchors.}
\label{sec:combining-anchors}
Trajectory Soup averages the selected branch anchors,
\begin{equation}
\boxed{
\theta_{\mathrm{Traj\text{-}Soup}}(N,t,K)
=\frac1N\sum_{n=1}^N\bar\theta_n(t,K)
=\frac1{NK}\sum_{n=1}^N\sum_{i\in\mathcal I_n(t,K)}\theta_n(i).
}
\label{eq:trajectory-soup}
\end{equation}
\textbf{Fig}.~\ref{fig:method-diagram} illustrates the procedure and places it beside the two single-stage merges it generalizes. Algorithm~\ref{alg:trajectory_soup} states the two averaging stages for the admitted branches. The two-stage construction is algebraically equivalent to uniformly averaging the selected $M=NK$ checkpoints and the step order can be reversed. The special case $N=1$ recovers selected intra-trajectory averaging, while $K=1$ averages the best checkpoint from each branch. Trajectory Soup introduces branch count as an additional axis of compute allocation. As serial training approaches saturation, we seek to distribute the budget across shorter compatible trajectories, whose complementary progress is consolidated into a single model. \textbf{Fig}.~\ref{fig:trajectory-soup-compute-allocation} further shows that allocating additional compute to complementary branches sustains downstream gains beyond the observed serial-saturation regime.

\section{Theoretical Insights}
\label{sec:theory}

We theoretically analyze how checkpoint selection and the two averaging stages in Trajectory Soup affect validation loss. We use a local quadratic model~\citep{mandt2017stochastic} and standard bias--variance identities~\citep{bishop2006pattern} to separate selection-dependent bias, persistent branch variation, and within-trajectory fluctuations.

\subsection{A Local Validation-Loss Model}
\label{sec:local-model}

Let $\mathcal L_Q(\theta)=\mathbb E_{z\sim Q}[\ell(\theta,z)]$ denote population validation loss, estimated by $\widehat{\mathcal L}_{\mathrm{val}}$ for checkpoint ranking in \textbf{Eq}.~\eqref{eq:topk-selection}. We assume that all analyzed checkpoints and averages lie in a region $\mathscr B$ around a stationary local minimizer $\theta^\star$, where $\mathcal L_Q(\theta)=\mathcal L^\star+\frac12\|\theta-\theta^\star\|_H^2$ with $\mathcal L^\star=\mathcal L_Q(\theta^\star)$, $H=\nabla^2\mathcal L_Q(\theta^\star)\succeq0$, and $\|x\|_H^2=x^\top Hx$. The compatibility screen of \textbf{Sec}.~\ref{sec:problem-formulation} and the interpolation profiles of \textbf{Sec}.~\ref{sec:motivation} motivate this local model, which underpins the stochastic analysis of constant-step SGD~\citep{mandt2017stochastic} and the flatness interpretation of weight averaging~\citep{wortsman2022model}.
\begin{theorem}[Local bias and variance decomposition]
\label{thm:local-loss}
For any random model $\vartheta$ supported in $\mathscr B$ with finite second moments,
\begin{equation}
\mathbb E[\mathcal L_Q(\vartheta)]-\mathcal L^\star=\frac12\|\mathbb E[\vartheta]-\theta^\star\|_H^2+\frac12\operatorname{tr}\!\left(H\operatorname{Cov}(\vartheta)\right).
\label{eq:expected-loss}
\end{equation}
\end{theorem}

The two terms measure displacement of the mean model and variation around it, respectively. Checkpoint selection can change both. The identity requires only finite moments and local compatibility, and makes no Gaussian or optimizer-specific assumption. The proof is given in Appendix~\ref{app:expected-loss}.

\subsection{Branchwise Decomposition of Validation Loss}
\label{sec:trajectory-effects}

We first identify the part of a branch's loss that intra-trajectory merging can reach. Let $\theta_{n,(j)}$ be the checkpoint of validation rank $j$ on branch $n$, and let $\bar\theta_n$ represent the checkpoint that averages its top-$K$ selection in \textbf{Eq}.~\eqref{eq:topk-selection}. Let $\mathscr F_n$ represent the persistent component of branch randomness, the part that a single run fixes and then carries along its entire path, so that conditioning on $\mathscr F_n$ holds that component fixed while retaining the variation among selected checkpoints. Theorem~\ref{thm:branch-loss} splits the loss of the anchor into the displacement of its mean, the within-trajectory variation that survives conditioning, and the persistent variation it inherits from $\mathscr F_n$.

\begin{theorem}[Branchwise decomposition of validation loss]
\label{thm:branch-loss}
Under the conditions of Theorem~\ref{thm:local-loss},
\begin{equation}
\mathbb E[\mathcal L_Q(\bar\theta_n)]-\mathcal L^\star
=\underbrace{\frac12\|\mathbb E[\bar\theta_n]-\theta^\star\|_H^2}_{B_n(K)}
+\underbrace{\frac12\mathbb E\!\left[\operatorname{tr}\!\left(H\operatorname{Cov}(\bar\theta_n\mid\mathscr F_n)\right)\right]}_{V_{\mathrm{intra},n}(K)}
+\underbrace{\frac12\operatorname{tr}\!\left(H\operatorname{Cov}(\mathbb E[\bar\theta_n\mid\mathscr F_n])\right)}_{V_{\mathrm{inter},n}(K)}.
\label{eq:branch-loss}
\end{equation}
\end{theorem}

All three contributions, $B_n(K)$, $V_{\mathrm{intra},n}(K)$, and $V_{\mathrm{inter},n}(K)$ are nonnegative. The bias $B_n(K)$ collects an offset shared by every branch, which neither averaging stage removes, together with offsets specific to the recipe $\psi_n$ and to the selection itself. Appendix~\ref{app:covariance} records that decomposition and the random-effects construction behind it. The two variance terms differ in what can reach them: $V_{\mathrm{intra},n}(K)$ is precisely what merging within a branch acts on, whereas $V_{\mathrm{inter},n}(K)$ is invisible to any operation confined to a single trajectory. Sampling the path more densely within one trajectory may reveal more about that single branch, but it gives no access to the cross-branch variation. This tradeoff explains how the intra-trajectory gains of \textbf{Sec}.~\ref{sec:motivation} can diminish, and how fast that term falls is what \textbf{Sec}.~\ref{sec:two-levels} quantifies. Appendix~\ref{app:proof-branch-loss} proves the theorem.

\subsection{Hierarchical Variance Reduction}
\label{sec:two-levels}
\label{sec:equal-weights}

To quantify the residual reduction, note that the term $V_{\mathrm{intra},n}(K)$ of \textbf{Eq}.~\eqref{eq:branch-loss} is the curvature-weighted covariance of the average, over the $K$ selected ranks, of the fluctuations of the selected checkpoints around their conditional means given $\mathscr F_n$. Temporal correlation and ranking-induced dependence therefore both limit how fast it decays. For branches with a common rank-1 residual contribution $V_{\mathrm{intra}}=V_{\mathrm{intra},n}(1)>0$, we write $V_{\mathrm{intra},n}(K)=V_{\mathrm{intra}}/K_{\mathrm{eff}}(K)$, whose effective count is a variance ratio: it reaches $K$ only when those fluctuations are uncorrelated with equal marginal contributions, and stays below $K$ whenever their pairwise contributions are nonnegative. Appendix~\ref{app:covariance} gives the residual covariance behind this ratio together with its bounds and limiting cases. Thus $K_{\mathrm{eff}}(K)$ settles the first averaging level, which operates inside a single branch. The second level averages the $N$ branch anchors themselves. The following theorem therefore introduces the corresponding bias term and collects both levels in one expression.

\begin{theorem}[Hierarchical variance reduction]
\label{thm:hierarchical}
Suppose the shared-shift condition holds and the selected anchors are independent under the fixed experimental setup. Define $B_{\mathrm{soup}}(K)=\frac12\left\|\frac1N\sum_{n=1}^N\mathbb E[\bar\theta_n]-\theta^\star\right\|_H^2$, with common contributions $V_{\mathrm{intra},n}(K)=V_{\mathrm{intra}}/K_{\mathrm{eff}}(K)$ and $V_{\mathrm{inter},n}(K)=V_{\mathrm{inter}}$.
The equally averaged model $\theta_{\mathrm{Traj\text{-}Soup}}=N^{-1}\sum_{n=1}^N\bar\theta_n$ of \textbf{Eq}.~\eqref{eq:trajectory-soup} satisfies
\begin{equation}
\boxed{\mathbb E[\mathcal L_Q(\theta_{\mathrm{Traj\text{-}Soup}})]-\mathcal L^\star=B_{\mathrm{soup}}(K)+\frac{V_{\mathrm{intra}}}{N K_{\mathrm{eff}}(K)}+\frac{V_{\mathrm{inter}}}{N}}
\label{eq:hierarchical}
\end{equation}
\end{theorem}

The bias $B_{\mathrm{soup}}(K)$ depends on the chosen branch collection. Temporal averaging reduces the intra-trajectory contribution through $K_{\mathrm{eff}}(K)$, while branch averaging reduces both stochastic contributions through $N$. \textbf{Tab}.~\ref{tab:averaging-methods} compares these effects. Appendix~\ref{app:proof-hierarchical} gives the proof.

\begin{table*}[t]
\centering
\small
\renewcommand{\arraystretch}{1.22}
\begin{tabular}{lccc}
\toprule
Method & Intra & Inter & Expected excess loss in the local quadratic model\\
\midrule
Rank-1 checkpoint & No & No
& $B_n(1)+V_{\mathrm{intra}}+V_{\mathrm{inter}}$\\
Intra-trajectory merging & Yes & No
& $B_n(K)+V_{\mathrm{intra}}/K_{\mathrm{eff}}(K)+V_{\mathrm{inter}}$\\
Inter-trajectory merging & No & Yes
& $B_{\mathrm{soup}}(1)+V_{\mathrm{intra}}/N+V_{\mathrm{inter}}/N$\\
Trajectory Soup & Yes & Yes
& $B_{\mathrm{soup}}(K)+V_{\mathrm{intra}}/(N K_{\mathrm{eff}}(K))+
V_{\mathrm{inter}}/N$\\
\bottomrule
\end{tabular}
\caption{Comparsive analysis of averaging operations under the local quadratic model and the homogeneous independent-anchor assumptions of Theorem~\ref{thm:hierarchical}.}
\label{tab:averaging-methods}
\end{table*}

\textbf{Variance optimality of uniform averaging.} Finally, the same decomposition reveals that the two \textit{uniform} averages of Trajectory Soup are variance-optimal rather than merely convenient. Consider replacing them with deterministic nonnegative weights over ranks and branches, each normalized to sum to one. Under the condition that every selected checkpoint of a branch shares the same aggregate curvature-weighted covariance with its selected set, and that the branch-level collections are independent, no such reweighting can attain a smaller stochastic contribution than \(V_{\mathrm{intra}}/(N K_{\mathrm{eff}}(K)) + V_{\mathrm{inter}}/N\) in Eq.~\eqref{eq:hierarchical}, and the uniform temporal and branch weights of Eq.~\eqref{eq:trajectory-soup} attain this bound. This condition permits correlated residuals and is therefore weaker than independence, but it is substantive, since equal marginal variances and positive correlations alone do not imply equal covariance sums. 

\section{Experiments}
\label{sec:experiments}

\subsection{Experimental Setup}
\label{sec:setup}

\textbf{Model and mid-training trajectory settings.}
We conduct our default experiments with Ling-3.0-Tiny~\citep{inclusionai2026ling3tiny}, a sparse MoE language model with 7.9B total and 1.3B parameters activated per token. We fork multiple independent trajectories from the same checkpoint. To induce controlled trajectory diversity, each branch differs from the default configuration along one or more of the following dimensions: data-shuffling seed, peak learning rate, global batch size, learning-rate schedule, and optimizer. The default branch horizon is $t=600$B tokens, which is sufficiently long for mid-training performance to approach saturation. Every branch saves a checkpoint each 25B tokens. Appendices~\ref{app:experiments-config} record the architecture, training recipe, and evaluation protocol.

\textbf{Merging baselines.}
Given a candidate pool, we compare four merging scopes: (i) \textbf{Single EXP Merge}, which combines checkpoints within a single experiment run (EXP); (ii) \textbf{Model Soup}, which combines the final checkpoint of each trajectory; (iii) \textbf{Full Soup}, which applies no selection and naively averages all checkpoints pooled from multiple trajectories, and which we therefore report as a selection ablation in \textbf{Sec}.~\ref{sec:ablations}; and (iv) our \textbf{Trajectory Soup}, which jointly selects and combines checkpoints both within and across trajectories. 

\textbf{Compute-budget comparisons.}
We distinguish compute-matched evaluation from compute scaling. In the \textbf{Limited} setting, $N$ trajectories share a total budget of $T=Nt$ tokens (each covering $T/N$ tokens) and are compared compute-matched against a single trajectory trained for the same $T$. Reported budgets denote aggregate consumption. In the \textbf{Extended} setting, we merge $N$ fully trained trajectories and increase $N$. Reported budgets denote per-trajectory consumption, testing whether merging converts an expanded budget ($t \to Nt$ tokens) into improved quality. 

\begin{figure*}[tb]
    \centering
    \includegraphics[width=\textwidth]{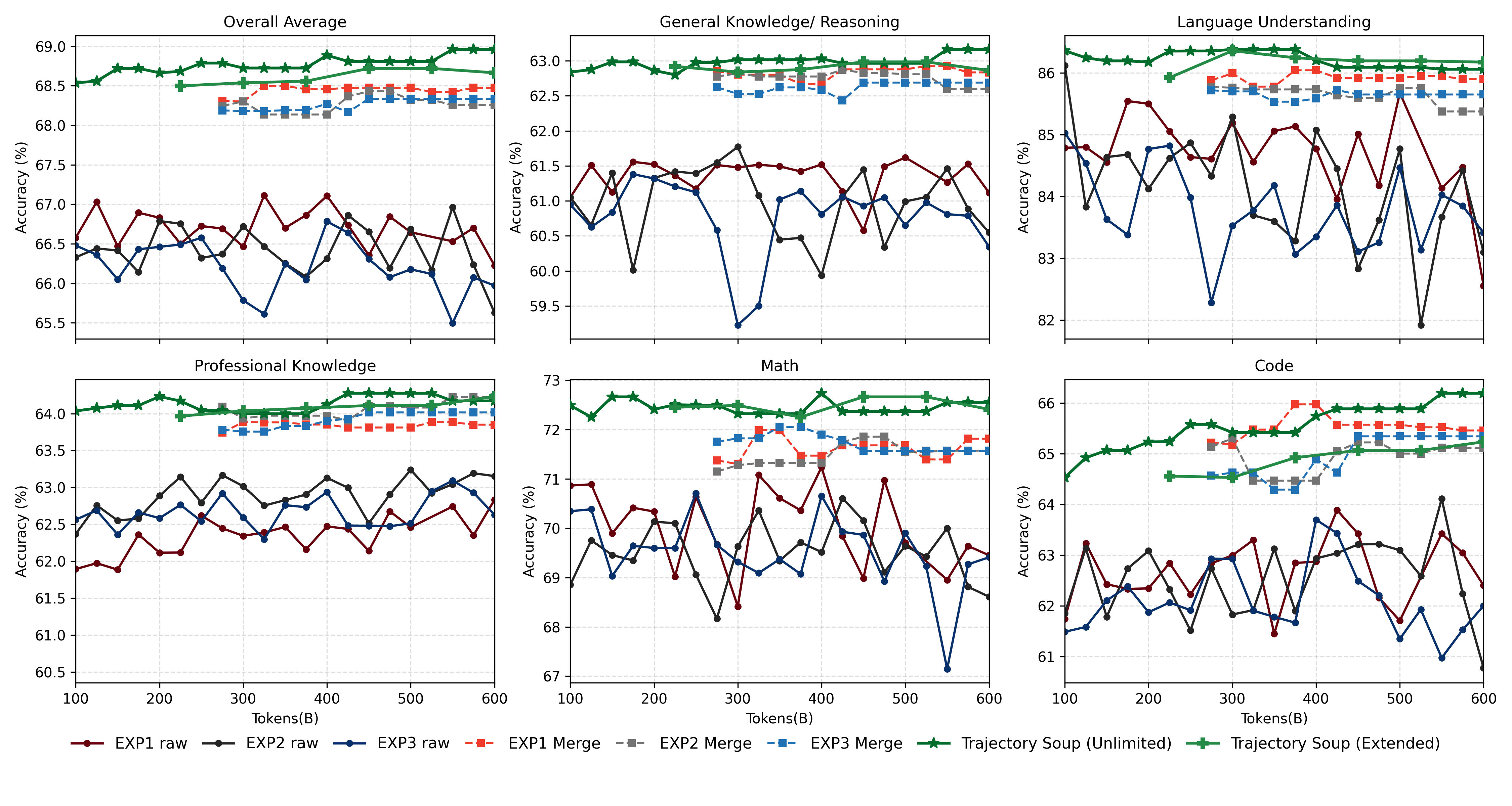}
    \caption{\textbf{Performance of different mid-training trajectories and merging strategies as the training-token budget increases.} Panels report the overall average accuracy and the five capability categories. Branch identities, merge sizes, and the Limited and Extended budget conventions follow \textbf{Sec}.~\ref{sec:setup}.}
    \label{fig:traj-soup-performance-curves}
\end{figure*}

\subsection{Main Results of Trajectory Soup}
\label{sec:main-results}

\textbf{Intra- and Inter-trajectory Merging Effects during Mid-training.} We systematically compare checkpoint-merging strategies during mid-training, focusing on how intra- and inter-trajectory effects interact across three independently trained trajectories: EXP1 (Baseline), EXP2 (Data Shuffle), and EXP3 (Hyperparameter Change). \textbf{Tab}.~\ref{tab:tiny-mid-train-main} summarizes the best candidate produced by each strategy. Inter-trajectory diversity yields further gains over intra-trajectory merging. In the compute-matched setting, Trajectory Soup (Limited) outperforms the strongest single-trajectory merge and Model Soup (Limited). In the full-trajectory setting, Trajectory Soup (Extended) achieves the best overall average of 68.96, exceeding the strongest single-trajectory merge and Model Soup (Extended). These results suggest that the gain cannot be attributed solely to either trajectory diversity or the number of merged checkpoints. Instead, effective scaling requires jointly exploiting inter-trajectory complementarity and selectively filtering checkpoints within each trajectory.
The complete performance curves over increasing training-token budgets are deferred to \textbf{Fig}.~\ref{fig:traj-soup-performance-curves}, which shows that the ordering below holds throughout mid-training.

\begin{table}[tb]
\centering
\caption{Performance comparison of different checkpoint-merging strategies on the mid-training evaluation suite. Each row reports the best evaluated configuration of its strategy; the corresponding checkpoint counts are recorded in Appendix~\ref{app:protocol}.}
\label{tab:tiny-mid-train-main}
\resizebox{0.95\linewidth}{!}{
\begin{tabular}{lcccccc}
\toprule
Base Model & \begin{tabular}[c]{@{}c@{}}General Knowledge\\\& Reasoning\end{tabular}
      & \begin{tabular}[c]{@{}c@{}}Language\\Modeling\end{tabular}
      & \begin{tabular}[c]{@{}c@{}}Professional\\Knowledge\end{tabular}
      & Math
      & Code
      & \begin{tabular}[c]{@{}c@{}}Overall\\Average\end{tabular} \\
\midrule
Single-Trajectory Merge                 & 62.80          & 85.86          & 63.70          & 72.33          & 65.36          & 68.55          \\
\midrule
Model Soup (Limited)       & 62.74          & 85.80          & 63.59          & 72.53          & 64.74          & 68.43          \\
Model Soup (Extended)      & 62.78          & 85.95          & 63.68          & 72.49          & 65.78          & 68.67          \\
\midrule
Trajectory Soup (Limited)  & 62.99          & \textbf{86.19} & 64.11          & \textbf{72.66} & 65.07          & 68.72          \\
Trajectory Soup (Extended) & \textbf{63.16} & 86.06          & \textbf{64.17} & 72.55          & \textbf{66.19} & \textbf{68.96} \\
\bottomrule
\end{tabular}
}
\end{table}

\begin{table}[tb]
\centering
\caption{Post-training performance after applying the same SFT procedure to mid-training checkpoints produced by different merging strategies. The merging protocols are those of \textbf{Tab}.~\ref{tab:tiny-mid-train-main}.}
\resizebox{0.95\linewidth}{!}{
\begin{tabular}{lccccccc}
\toprule
Instruct Model & Math
      & Code
      & Knowledge
      & Reasoning
      & \begin{tabular}[c]{@{}c@{}}Instruction\\Following\end{tabular}
      & \begin{tabular}[c]{@{}c@{}}Function\\Call\end{tabular}
      & \begin{tabular}[c]{@{}c@{}}Overall\\Average\end{tabular} \\
\midrule
Single-Trajectory Merge                 & 68.30          & 46.58          & \textbf{70.51} & 62.96          & 59.46          & 50.67          & 61.11          \\
\midrule
Model Soup (Limited)       & 68.50          & 47.59          & 69.36          & 63.56          & 58.79          & 51.52          & 61.05          \\
Model Soup (Extended)      & 68.63          & 46.76          & 69.65          & 64.41          & 59.29          & 51.63          & 61.23          \\
\midrule
Trajectory Soup (Limited)  & \textbf{69.44} & \textbf{48.33} & 69.81          & 63.60          & 58.88          & 51.49          & 61.39          \\
Trajectory Soup (Extended) & 68.53          & 47.50          & 69.26          & 64.48          & \textbf{60.38} & \textbf{52.85} & \textbf{61.52} \\
\bottomrule
\end{tabular}
}
\label{tab:tiny-post-train-main}
\end{table}

\textbf{Long-term Implications for Post-training.} The preceding experiments establish that Trajectory Soup achieves the strongest aggregate performance at the mid-training stage, which naturally raises the question of whether this advantage persists after post-training. To investigate this, we use the mid-training checkpoints produced by each merging strategy as base models, apply an identical SFT recipe, and evaluate the post-training capabilities, with results reported in \textbf{Tab}.~\ref{tab:tiny-post-train-main}. The advantage of Trajectory Soup is largely preserved after SFT, indicating that our method provides a stronger initialization whose benefits transfer to subsequent post-training. 

\begin{table}[tb]
\centering
\caption{\textbf{Comparison of single-trajectory merging and cross-trajectory Trajectory Soup during mid-training of the smaller 2B-parameter MoE model.} The Limited and Extended budget conventions follow \textbf{Sec}.~\ref{sec:setup}.}
\resizebox{0.95\linewidth}{!}{
\begin{tabular}{lcccccc}
\toprule
Model & \begin{tabular}[c]{@{}c@{}}General Knowledge\\\& Reasoning\end{tabular}
      & \begin{tabular}[c]{@{}c@{}}Language\\Modeling\end{tabular}
      & \begin{tabular}[c]{@{}c@{}}Professional\\Knowledge\end{tabular}
      & Math
      & Code
      & \begin{tabular}[c]{@{}c@{}}Overall\\Average\end{tabular} \\
\midrule
Single-Trajectory Merge                 & 47.71          & 72.85          & 43.41          & 54.21          & \textbf{35.96} & 49.55          \\
\midrule
Trajectory Soup (Limited)  & 47.88          & 73.32          & 43.70          & 54.13          & 35.77          & 49.64          \\
Trajectory Soup (Extended) & \textbf{48.35} & \textbf{73.98} & \textbf{44.37} & \textbf{54.87} & 35.51          & \textbf{50.07} \\
\bottomrule
\end{tabular}
}
\label{tab:e32a8-mid-train-main}
\end{table}

\subsection{Empirical Analysis of Trajectory Soup}

\subsubsection{Robustness across models and learning-rate schedules.}
To examine whether the benefits of Trajectory Soup depend on a particular base model or learning-rate schedule, we repeat our experiments on a smaller 2B-parameter MoE model with 32 experts, activating 8 experts per token, using a warmup-stable-decay (WSD) learning-rate schedule~\citep{hu2024minicpm}. We compare intra-trajectory averaging within each trajectory against our Trajectory Soup. For each strategy, \textbf{Tab}.~\ref{tab:e32a8-mid-train-main} reports the merging results. These results reproduce the aggregate trend observed under our default configuration. Under the compute-matched Limited setting, Trajectory Soup achieves an overall average of 49.64, slightly outperforming the strongest intra-trajectory baseline, despite each constituent trajectory receiving only half of the total token budget. This suggests that the complementarity introduced by trajectory diversity can compensate for reduced training-token coverage along each individual trajectory. Under the Extended setting, Trajectory Soup further improves the overall average to 50.07, exceeding the best intra-trajectory result. Together with the results obtained under the default setup, these findings indicate that the gains from Trajectory Soup are not confined to a single model scale or scheduling choice, but remain reproducible on a smaller MoE model trained with WSD.

\begin{table}[tb]
\centering
\caption{Ablation of merging coefficients and checkpoint-selection strategies. Every row follows the Trajectory Soup (Extended) setting of \textbf{Tab}.~\ref{tab:tiny-mid-train-main}, with Baseline the default configuration of \textbf{Sec}.~\ref{sec:setup}.}
\resizebox{0.95\linewidth}{!}{
\begin{tabular}{lllcccccc}
\toprule
Ablation & Coefficient & Selection
      & \begin{tabular}[c]{@{}c@{}}General Knowledge\\\& Reasoning\end{tabular}
      & \begin{tabular}[c]{@{}c@{}}Language\\Modeling\end{tabular}
      & \begin{tabular}[c]{@{}c@{}}Professional\\Knowledge\end{tabular}
      & Math
      & Code
      & \begin{tabular}[c]{@{}c@{}}Overall\\Average\end{tabular} \\
\midrule
Baseline     & EQUAL & Top-K each    & 63.16          & 86.06          & 64.17          & 72.55          & \textbf{66.19} & \textbf{68.96} \\
\midrule
\multirow{3}{*}{Coefficient}
             & 1SQRT & Top-K each    & \textbf{63.22} & 85.87          & 64.21          & 72.26          & 66.00          & 68.85          \\
             & RANK  & Top-K each    & 63.12          & 85.91          & 64.16          & 72.24          & 65.94          & 68.80          \\
             & RSQRT & Top-K each    & 63.04          & 86.03          & 64.26          & \textbf{72.59} & 65.97          & 68.89          \\
\midrule
\multirow{3}{*}{Selection}
             & EQUAL & Global top-NK & 62.99          & \textbf{86.24} & 64.18          & 72.46          & 65.43          & 68.76          \\
             & EQUAL & Tail-K per branch & 62.78          & 84.97          & \textbf{64.29} & 71.77          & 65.25          & 68.35          \\
             & EQUAL & All checkpoints & 62.94          & 85.88          & 64.26          & 71.79          & 65.70          & 68.60          \\
\bottomrule
\end{tabular}
}
\label{tab:tiny-coefficient-selection-ablation}
\end{table}

\subsubsection{Ablation Studies on the Trajectory Soup Configuration}
\label{sec:ablations}

We ablate three key design choices in Trajectory Soup: (1) how merge coefficients are assigned, (2) how checkpoints are selected, and (3) how the checkpoint budget is allocated across trajectories.

\textbf{Effect of merge coefficients.}
We compare uniform averaging against three rank-based weighting schemes, detailed in Appendix~\ref{app:merge-coefficients}. As shown in \textbf{Tab}.~\ref{tab:tiny-coefficient-selection-ablation}, none of them improves upon the equal-weight baseline. Once checkpoint quality is controlled through selection, more elaborate coefficient design provides little additional benefit, so uniform averaging is the strongest evaluated scheme and avoids introducing additional hyperparameters.

\textbf{Effect of checkpoint selection.}
We next fix merge coefficients to be uniform and consider three alternative selection strategies, also defined in Appendix~\ref{app:merge-coefficients}. All three underperform the default \emph{Top-$K$ each} strategy. These results identify both checkpoint quality and balanced trajectory representation as important selection criteria.

\textbf{Importance of Balanced Allocation across Trajectories.} 
\textbf{Fig}.~\ref{fig:traj-soup-asym-allocation} tests balanced representation directly: we fix the total number of merged checkpoints to 16 and vary their allocation between EXP1 and EXP3 from $2{:}14$ to $14{:}2$. Performance follows an approximately inverted-U-shaped profile centered on the balanced allocation, with the symmetric $8{:}8$ configuration attaining the highest overall average. Sufficiently balanced participation from both trajectories is therefore necessary to capture their complementary information, supporting symmetric allocation as a robust default.

Overall, the three ablations consistently indicate that the effectiveness of Trajectory Soup is driven by the combination of trajectory diversity, checkpoint quality, and balanced representation, rather than by sophisticated coefficient design. Among the configurations evaluated, the simple combination of uniform averaging, top-$K$ selection within each trajectory, and symmetric cross-trajectory allocation provides the strongest aggregate performance while requiring minimal additional tuning.


%

\begin{figure*}[t]
    \centering
    \begin{minipage}[t]{0.48\textwidth}
        \centering
        \includegraphics[width=\linewidth]{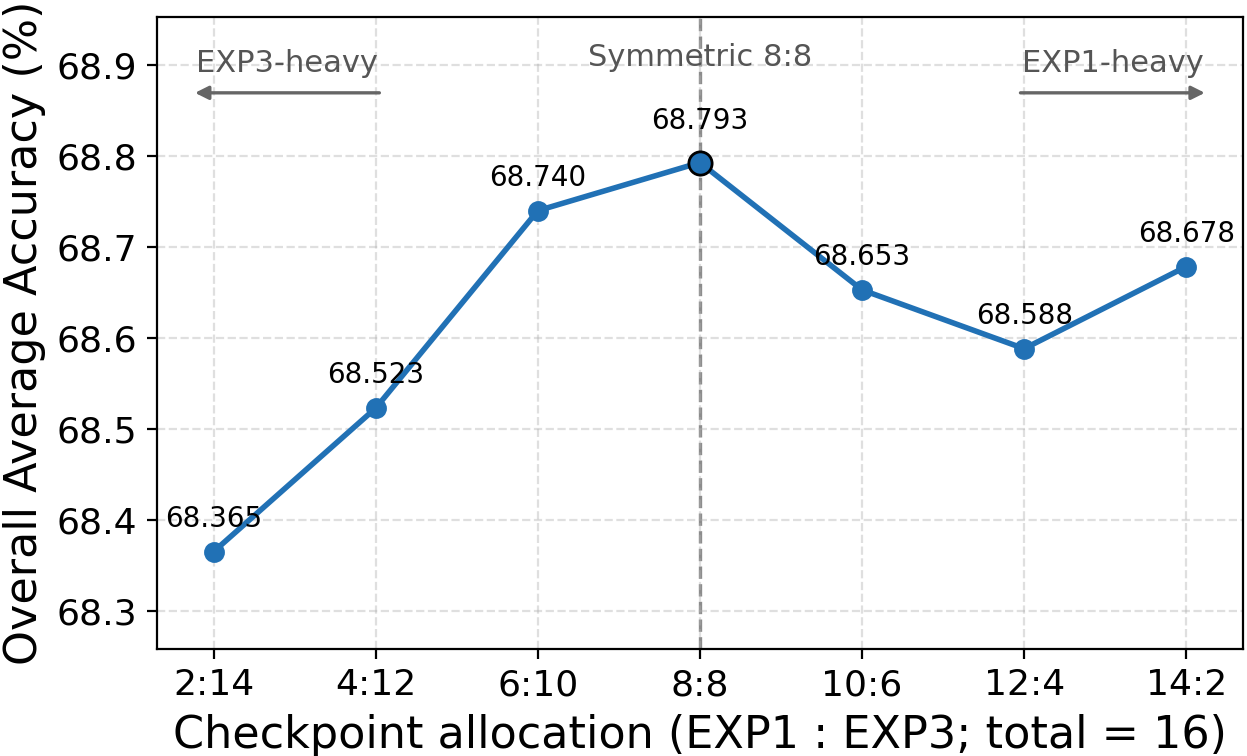}
        \caption{\textbf{Effect of asymmetric merging allocation on Trajectory Soup performance.}}
        \label{fig:traj-soup-asym-allocation}
    \end{minipage}
    \hfill
    \begin{minipage}[t]{0.48\textwidth}
        \centering
        \includegraphics[width=\linewidth]{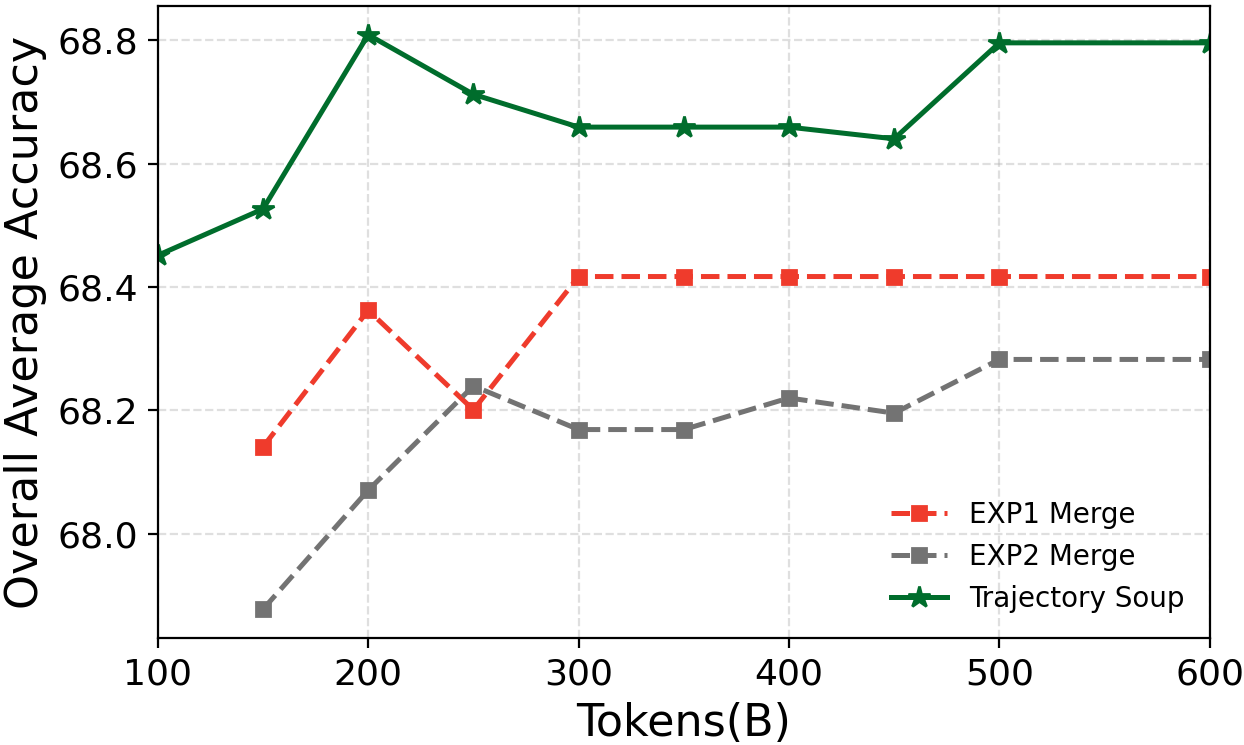}
        \caption{\textbf{Performance of densely sampled single run merging and equally sized sparse Trajectory Soup.}}
        \label{fig:traj-soup-sample-ablation}
    \end{minipage}
\end{figure*}

\subsubsection{Trajectory Diversity over Sampling Density}
\label{sec:sampling-density}

The preceding experiments show that Trajectory Soup consistently outperforms intra-trajectory merging throughout mid-training. However, this advantage could conceivably arise from a trivial confound: \textit{Trajectory Soup may simply draw from a larger checkpoint pool, thereby offering more candidates or a higher effective sampling density.} To isolate the contribution of trajectory diversity, we conduct a controlled comparison in which both the candidate-pool size and the number of merged checkpoints are strictly matched. Specifically, we fix the merge size to $M=12$ and compare densely sampled intra-trajectory merging with sparsely sampled cross-trajectory merging. In our experiment, EXP1 Merge and EXP2 Merge sample checkpoints every 12.5B tokens along their respective trajectories and uniformly average the top 12 checkpoints. In contrast, Trajectory Soup halves the per-trajectory sampling density by sampling checkpoints every 25B tokens from each trajectory and uniformly averages the top six checkpoints from each, again totaling 12. Thus, at a per-trajectory token coverage of $t$, both settings draw from exactly $t/12.5\mathrm{B}$ candidate checkpoints and merge the same number of checkpoints, and the only systematic difference is whether these candidates originate from a single trajectory or from two independently optimized ones. As shown in \textbf{Fig}.~\ref{fig:traj-soup-sample-ablation}, Trajectory Soup outperforms both intra-trajectory baselines at every comparable token budget, demonstrating that its gains cannot be attributed to a larger candidate pool or denser checkpoint sampling. This controlled comparison therefore provides direct evidence that the gains of Trajectory Soup stem from complementary information accumulated along independently evolved trajectories, rather than from increased checkpoint availability.

\subsubsection{Scaling the Number of Trajectories and Merged Checkpoints}
\label{sec:scaling}
We further investigate how Trajectory Soup scales along two axes: the number of participating trajectories $N$ and the number of top-ranked checkpoints $K$ contributed by each. For every $(N,K)$ configuration we uniformly average the top $K$ checkpoints of each trajectory, and \textbf{Fig}.~\ref{fig:traj-soup-trajectory-scaling} reports the Overall Average against the total number of merged checkpoints, revealing two patterns. First, \emph{trajectory diversity, rather than checkpoint depth, drives performance}: the best attainable score rises monotonically from $68.79$ with two trajectories to $69.08$ with five, while the increments steadily shrink. This is consistent with the standard interpretation of weight averaging, where uniform averaging cancels trajectory-specific errors only when the merged solutions reside in a shared low-loss basin while making quasi-independent mistakes, so a new trajectory helps most when it expands coverage of the basin, and its marginal value fades as the covered region saturates. Second, \emph{the soup benefits most from being selective}: all groups peak within a narrow window of roughly ten to sixteen merged checkpoints, beyond which enlarging the pool brings oscillation or steady decline, with the four-trajectory soup eventually falling below the two-trajectory optimum. This makes Trajectory Soup practical to scale, adding trajectories reliably raises the performance ceiling, and retaining only a handful of top checkpoints per trajectory is sufficient to nearly reach it.

\subsubsection{Directional Diversity Predicts Interpolation Gain}
\label{sec:direction-gain}

\begin{figure*}[t]
    \centering
    \begin{minipage}[t]{0.48\textwidth}
        \centering
        \includegraphics[width=\linewidth]{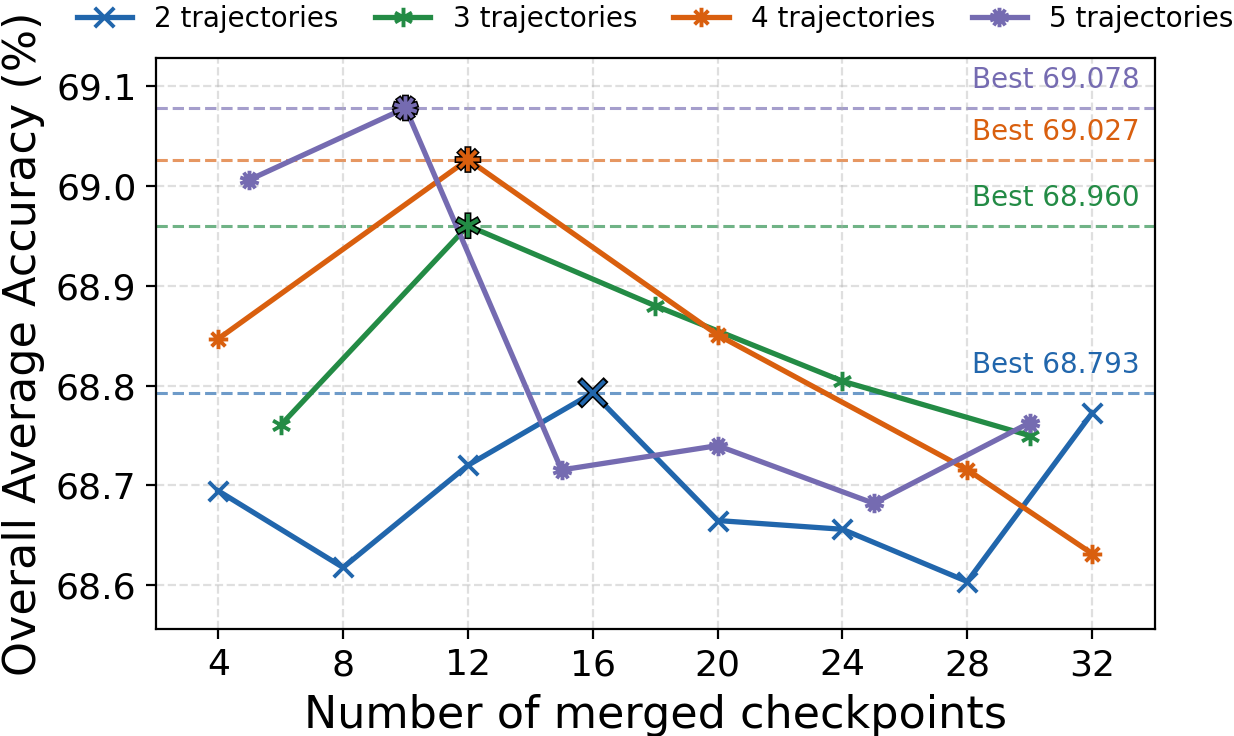}
        \caption{\textbf{Scaling of multi-trajectory Top-$K$ Trajectory-Soup with the number of merged checkpoints.}}
        \label{fig:traj-soup-trajectory-scaling}
    \end{minipage}
    \hfill
    \begin{minipage}[t]{0.48\textwidth}
        \centering
        \includegraphics[width=\linewidth]{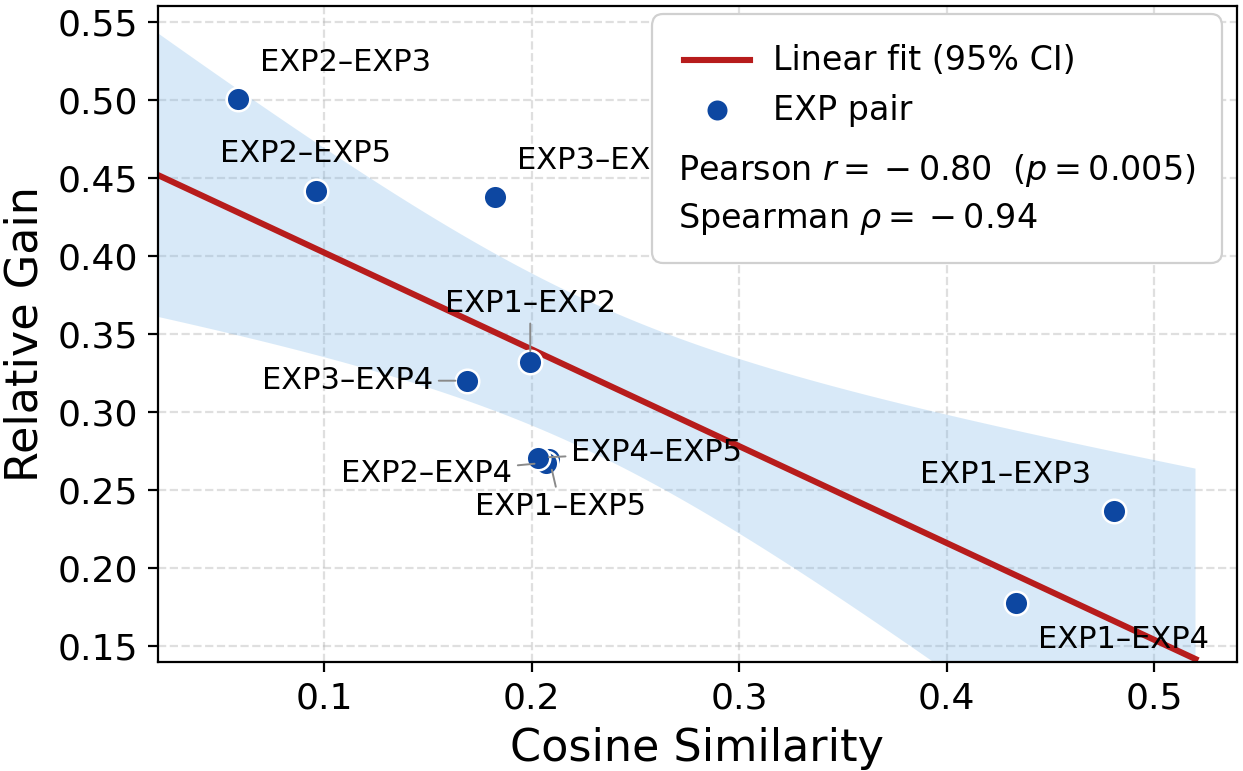}
        \caption{\textbf{Weaker directional alignment predicts larger interpolation gains.}}
        \label{fig:direction-gain}
    \end{minipage}
\end{figure*}

The scaling results show that adding trajectories improves the merged model, but they do not identify \emph{which} branches contribute the improvement or why. We therefore test whether pairwise branch geometry predicts the benefit of merging. \textbf{Fig}.~\ref{fig:direction-gain} plots the cosine similarity of the two leading directions against the relative gain of interpolating the pair's endpoints on the overall evaluation score. Both quantities are defined in Appendix~\ref{app:geometry}. The association is strongly negative and nearly monotone, with Pearson $r{=}{-}0.80$ ($p{=}0.005$) and Spearman $\rho{=}{-}0.94$: the near-orthogonal pairs gain the most and the most aligned pairs the least. Directional diversity therefore explains the scaling trend above and supplies a cheaply computable criterion, available before any merging is performed, for selecting which branches to merge.
\section{Related Work}
\label{sec:related}

\subsection{Mid-Training and Horizon-Extensible Scaling}
\label{sec:related-scaling}

Classical language-model scaling laws characterize the dependence of performance on model size, training data, and compute. Compute-optimal training rules balance parameter count and token consumption~\citep{hoffmann2022training}. In data-constrained regimes, the effective value of repeated tokens decreases as repetition grows~\citep{muennighoff2023scaling}. These results motivate distinguishing additional token consumption from the useful progress obtained from it.
Training schedules provide flexibility in extending a run. MiniCPM introduces the Warmup--Stable--Decay schedule, whose stable phase supports continued training followed by a cooldown~\citep{hu2024minicpm}. The river-valley analysis of WSD explains how a large learning rate can support progress along a slowly varying direction while maintaining fluctuations in sharper directions~\citep{wen2024understanding}. Complementing this optimization perspective, mid-training studies examine how mixtures of general and specialized data improve the starting point for post-training~\citep{liu2025midtraining}. Our study adds trajectory count to the allocation problem and examines selective weight-space aggregation under both matched and expanded training budgets.

\subsection{Checkpoint Averaging Along a Single Trajectory}
\label{sec:related-checkpoint-averaging}

Stochastic Weight Averaging combines iterates from constant or cyclical learning-rate training and often reaches flatter solutions with improved generalization~\citep{izmailov2018averaging}. Trainable Weight Averaging instead learns the averaging coefficients by optimizing in the low-dimensional subspace spanned by the candidate checkpoints rather than fixing them in advance~\citep{li2023trainable}. In language-model pre-training, \citet{sanyal2023early} show that averaging can improve convergence and downstream performance, with larger gains under high learning rates and sufficiently spaced checkpoints. Pre-trained Model Averaging studies checkpoint merging across dense and MoE language models and relates stable-phase averages to annealing behavior~\citep{li2025model}. WSM develops the connection through the effective coefficients assigned to historical updates~\citep{tian2026wsm}.
Extra-Merge identifies approximately one-dimensional structure in late-stage merged trajectories and uses this structure for extrapolation~\citep{zhou2026extra}. These approaches establish the value of temporal aggregation. We build on this foundation by combining selected temporal averages across independently evolved branches and studying how this changes the allocation of training compute.

\subsection{Combining Multiple Training Trajectories}
\label{sec:related-multi-trajectory}

Prediction ensembles retain multiple models at inference time. Weight-space combinations can consolidate compatible models into one parameter set. Model Soups demonstrates effective averaging of models fine-tuned from a common initialization with different hyperparameters~\citep{wortsman2022model}. 
In language modeling, Branch-Train-Merge trains domain-specialized expert language models that can be ensembled or averaged, and studies their performance under controlled training cost~\citep{li2022branch}. Branch-Train-MiX incorporates independently trained feedforward components into MoE layers, followed by training that learns routing~\citep{sukhbaatar2024branch}. 
The geometry underlying such combinations is related to mode connectivity. \citet{garipov2018loss} construct low-loss curves between trained solutions, while \citet{frankle2020linear} study linear connectivity under different realizations of SGD noise after a shared training history. Our interpolation analysis examines the straight segments between mid-training anchors, which are directly relevant to arithmetic parameter averaging.

\section{Discussion and Limitations}
\label{sec:discussion}

\textbf{Scope of the compute accounting.} Our allocation comparison holds processed training tokens fixed at a common architecture and sequence length, which isolates the allocation question but charges nothing for the machinery around it. Checkpoint storage, the validation passes used for ranking, and the search over how many trajectories and checkpoints to combine all consume resources that a complete development-cost account would attribute to each strategy, and accelerator utilization differs between one long run and several concurrent shorter ones. Making these costs explicit is the natural next step, because a cost model that prices a trajectory against a token would let the allocation be chosen by optimization rather than by sweeping candidates.

\textbf{What the analysis assumes.} The theory is exact under a local quadratic model of the validation loss and inherits whatever error that model carries over the region spanned by the selected checkpoints and their averages. The geometric diagnostics are complementary but weaker. Direction angles characterize differences in a chosen parameter-space representation, and interpolation profiles probe particular averages, so neither establishes that the region containing all branches is jointly well behaved. Both are also measured only after the branches exist. A more useful account would predict which recipe perturbations generate error components that averaging can remove, turning diversity from a property screened after training into one designed before it.

\textbf{Where averaging stops paying.} Averaging cannot remove error that the branches share, so strongly correlated trajectories leave a floor that no amount of merging crosses, and perturbations chosen too aggressively push branches into regions where the merged model is worse than its members. The same tension governs the budget, since at a fixed total more branches means shorter runs, and admitting more checkpoints per branch eventually admits weak ones. Trajectory Soup therefore depends on selecting compatible and sufficiently trained branches together with a suitable per-branch checkpoint count, which we currently fix by validation search. Learning this allocation during training, or adapting it as branches diverge, would remove the remaining hand-tuning and is where we expect the largest gains to come from.

\section{Conclusion}
\label{sec:conclusion}

Mid-training has so far been scaled by making a single optimization trajectory longer, and that axis eventually stops paying. This paper treats the number of trajectories as an allocation variable of the same kind as parameters and tokens, and works out what it takes to spend a budget that way. Trajectory Soup forks independent branches from a shared checkpoint, selects the strongest checkpoints inside each branch, and consolidates the resulting anchors into a single model whose architecture, tokenizer, and inference cost are unchanged. A local bias and variance analysis accounts for the design, showing why the two averaging levels address different parts of the error, why uniform weights are the natural default at both levels, and why checkpoint selection carries a bias that bounds how many checkpoints are worth merging.
Weight-space consolidation is a general way to convert parallel exploration into one deliverable model, rather than a technique specific to mid-training. The same structure arises wherever a budget can be split across runs that share an initialization, including continued pretraining, domain adaptation, and the increasingly long post-training pipelines that follow. Realizing that generality calls for understanding which perturbations create removable error, how to price a trajectory against a token, and how far compatibility survives objectives that carry models far from their shared origin. Treating trajectory count as a standard dimension of the compute-scaling question, rather than a special case of model merging, is in our view the more productive framing for the work that follows.

\clearpage

\bibliographystyle{plainnat}
\bibliography{references}

\clearpage
\appendix
\renewcommand{\theequation}{A\arabic{equation}}
\renewcommand{\theHequation}{A\arabic{equation}}
\setcounter{equation}{0}


\section{Additional Theory for Hierarchical Merging}
\label{sec:theory-appendix}

This appendix records the modeling detail behind the main-text results. We make the random-effects assumptions of \textbf{Sec}.~\ref{sec:trajectory-effects} explicit and allow heterogeneous or correlated branches, examine how the selected checkpoint count trades residual reduction against bias, restore the fixed-budget coupling between branch count and branch length, state the variance optimality of the two uniform averages that \textbf{Sec}.~\ref{sec:equal-weights} reports in prose, and give the exact covariance condition it requires. Every loss identity below is exact under the local quadratic model of \textbf{Sec}.~\ref{sec:local-model}. Throughout Appendices~\ref{sec:theory-appendix} and \ref{app:derivations} we abbreviate the two anchor means of \textbf{Sec}.~\ref{sec:trajectory-effects} as $\mu_n(K)=\mathbb E[\bar\theta_n]$ and $m_n(K)=\mathbb E[\bar\theta_n\mid\mathscr F_n]$. To keep the account compact, we develop these extensions as narrative rather than as separate numbered claims, apart from the corollary of Appendix~\ref{app:equal-weights}, and give their derivations in place. The three theorems stated in the main text and that corollary are proved in Appendix~\ref{app:derivations}.

\subsection{Conditional Random Effects and General Covariances}
\label{app:covariance}

Theorem~\ref{thm:branch-loss} uses total covariance without requiring a particular training algorithm. Its shared-shift specialization can be represented by a local random-effects model. For fixed $t$, assume
\begin{equation}
\mathbb E[\theta_{n,(j)}\mid\mathscr F_n]
=\theta^\star+b_{\mathrm{com}}+b_{\mathrm{rec},n}+h_{n,j}+u_n,
\qquad\mathbb E[u_n]=0,
\qquad h_{n,1}=0,
\label{eq:conditional-mean-model}
\end{equation}
where $u_n$ is measurable with respect to $\mathscr F_n$ and common to all selected ranks on that branch, and the deterministic offsets absorb selection-induced changes in mean. Here $b_{\mathrm{com}}$ is a finite-horizon offset shared by every branch, $b_{\mathrm{rec},n}$ is the systematic offset of recipe $\psi_n$, and $h_{n,j}$ is the rank-$j$ contribution to the selection-induced shift. Write $\epsilon_{n,j}=\theta_{n,(j)}-\mathbb E[\theta_{n,(j)}\mid\mathscr F_n]$ for the conditionally centered fluctuation that \textbf{Sec}.~\ref{sec:two-levels} describes in words. Averaging over the selected ranks then gives $h_n(K)=K^{-1}\sum_{j=1}^K h_{n,j}$ and $\bar\epsilon_n(K)=K^{-1}\sum_{j=1}^K\epsilon_{n,j}$, so the mean displacement described in \textbf{Sec}.~\ref{sec:trajectory-effects} reads
\begin{equation}
\mu_n(K)-\theta^\star=b_{\mathrm{com}}+b_{\mathrm{rec},n}+h_n(K),\qquad h_n(1)=0.
\label{eq:mean-displacement}
\end{equation}
A single $u_n$ shared by all selected ranks is exactly the shared-shift condition that Theorem~\ref{thm:hierarchical} assumes: it makes $m_n(K)=\mu_n(K)+u_n$ and therefore $V_{\mathrm{inter},n}(K)=\frac12\operatorname{tr}(H\operatorname{Cov}(u_n))$ independent of $K$, so selection moves the anchor only through $h_n(K)$ and the residual average. This is a modeling assumption; if the conditional shift changes across ranks, the general term $V_{\mathrm{inter},n}(K)$ of Theorem~\ref{thm:branch-loss} retains that dependence.

The homogeneous independent-anchor case is what exposes the scaling factors, but a general branch collection needs the full anchor cross-covariance $\Gamma_{nn'}(K)=\operatorname{Cov}(\bar\theta_n,\bar\theta_{n'})$ together with its scalar form $G_{nn'}(K)=\frac12\operatorname{tr}(H\Gamma_{nn'}(K))$. Two properties follow directly. First, $G(K)$ is symmetric because $\Gamma_{n'n}(K)=\Gamma_{nn'}(K)^\top$ and $H$ is symmetric, and it is positive semidefinite because $a^\top G(K)a=\frac12\mathbb E[\|\sum_{n}a_n(\bar\theta_n-\mu_n(K))\|_H^2]\geq0$ for every $a\in\mathbb R^N$. Second, the covariance of the branch-reweighted model $\theta_\alpha=\sum_n\alpha_n\bar\theta_n$, which keeps the uniform temporal average inside every branch, is $\sum_{n,n'}\alpha_n\alpha_{n'}\Gamma_{nn'}(K)$, so Theorem~\ref{thm:local-loss} gives, for fixed nonnegative weights summing to one,
\begin{equation}
\mathbb E[\mathcal L_Q(\theta_\alpha)]-\mathcal L^\star=B_\alpha(K)+\alpha^\top G(K)\alpha,
\qquad B_\alpha(K)=\tfrac12\|\mathbb E[\theta_\alpha]-\theta^\star\|_H^2.
\label{eq:general-weighted-loss}
\end{equation}
Setting every weight to $1/N$ turns the quadratic form into a double sum over the branch pairs,
\begin{equation}
\mathbb E[\mathcal L_Q(\theta_{\mathrm{Traj\text{-}Soup}})]-\mathcal L^\star
=B_{\mathrm{soup}}(K)+\frac1{2N^2}\sum_{n,n'=1}^N\operatorname{tr}\!\left(H\Gamma_{nn'}(K)\right).
\label{eq:general-cov-loss}
\end{equation}

Expanding the anchor as $\bar\theta_n=\theta^\star+b_{\mathrm{com}}+b_{\mathrm{rec},n}+h_n(K)+u_n+\bar\epsilon_n(K)$ resolves each block into three contributions: the covariance of the branch offsets $u_n$, the covariance of the averaged residuals $\bar\epsilon_n(K)$, and the two cross terms between an offset of one branch and the averaged residual of the other. Because $u_n$ is $\mathscr F_n$-measurable, conditional centering removes the cross terms on the diagonal and leaves $G_{nn}(K)=V_{\mathrm{intra},n}(K)+V_{\mathrm{inter},n}$, while placing no restriction on the cross terms for $n\ne n'$. Independent selected anchors make the off-diagonal $\Gamma_{nn'}(K)$ vanish, which reduces \textbf{Eq}.~\eqref{eq:general-cov-loss} to the heterogeneous form $B_{\mathrm{soup}}(K)+N^{-2}\sum_{n=1}^N[V_{\mathrm{intra},n}(K)+V_{\mathrm{inter},n}]$. Independence of the full random-effects collections is a stronger condition that makes every off-diagonal block vanish separately.

Equal scalar loss-weighted contributions therefore suffice for Theorem~\ref{thm:hierarchical}, and the full covariance matrices may still differ. Independent training followed by selection performed separately within each branch preserves anchor independence when the shared setup is fixed, whereas joint selection or a coupled compatibility decision leaves the off-diagonal blocks in place and \textbf{Eq}.~\eqref{eq:general-cov-loss} applies. Those blocks also bound what branch averaging can achieve, because a collection with positive average off-diagonal contribution drives the branch term toward that average rather than toward zero as $N$ grows, which is the correlated error floor discussed in Appendix~\ref{sec:limitations}.

Weighting the anchors unequally is worthwhile exactly when their contributions differ. For independent anchors with $v_n(K)=V_{\mathrm{intra},n}(K)+V_{\mathrm{inter},n}>0$, the variance term of \textbf{Eq}.~\eqref{eq:general-weighted-loss} is $\sum_n v_n(K)\alpha_n^2$, and Cauchy--Schwarz applied to $\sum_n\bigl(\sqrt{v_n(K)}\alpha_n\bigr)v_n(K)^{-1/2}=1$ gives $\sum_n v_n(K)\alpha_n^2\geq(\sum_n v_n(K)^{-1})^{-1}$, with equality only when $\alpha_n$ is proportional to $v_n(K)^{-1}$. The unique minimizer is thus the inverse-variance weighting $\alpha_n^{\mathrm{var}}(K)=v_n(K)^{-1}/\sum_{n'}v_{n'}(K)^{-1}$. If instead $G(K)$ is exchangeable with common diagonal $v$ and common off-diagonal $c$, then $\alpha^\top G(K)\alpha=c+(v-c)\sum_n\alpha_n^2$, where $v-c\geq0$ follows from positive semidefiniteness on vectors orthogonal to $\mathbf 1$, so equal weights minimize the variance and do so uniquely when $v>c$. The same exchangeable argument transfers to the residual weights over the $K$ ranks within one branch once the branch covariance matrix is replaced by the rank covariance matrix, while rank-dependent means continue to act through the bias instead. Both statements concern fixed weights, and estimating weights from the same realized checkpoints would require accounting for the joint randomness of weights and anchors.

The effective checkpoint count summarizes the rank covariance in a single scalar, so its interpretation needs both the marginal variation and the cross-checkpoint dependence. For a fixed branch, write $c_{jj'}=\frac12\operatorname{tr}(H\operatorname{Cov}(\epsilon_{n,j},\epsilon_{n,j'}))$ and assume the common rank-1 value $c_{jj}=V_{\mathrm{intra}}>0$. Expanding the covariance of $\bar\epsilon_n(K)$ and taking its loss-weighted trace gives
\begin{equation}
V_{\mathrm{intra},n}(K)=\frac1{K^2}\sum_{j,j'=1}^K c_{jj'},
\qquad
K_{\mathrm{eff}}(K)=\frac{K^2V_{\mathrm{intra}}}{\sum_{j,j'=1}^K c_{jj'}}.
\label{eq:effective-count-details}
\end{equation}
Since the residuals are mean zero, $c_{jj'}=\frac12\mathbb E[(H^{1/2}\epsilon_{n,j})^\top(H^{1/2}\epsilon_{n,j'})]$, and Cauchy--Schwarz bounds $|c_{jj'}|\leq\sqrt{c_{jj}c_{j'j'}}=V_{\mathrm{intra}}$. If every entry is nonnegative, the covariance sum lies between $KV_{\mathrm{intra}}$ and $K^2V_{\mathrm{intra}}$, so $1\leq K_{\mathrm{eff}}(K)\leq K$, with the upper bound attained exactly when the off-diagonal entries vanish and strict as soon as their sum is positive. Independent residuals attain that bound. Negative covariances can instead push $K_{\mathrm{eff}}(K)$ above $K$, heterogeneous ranks invalidate the equal-marginal premise, and when $V_{\mathrm{intra},n}(K)=0$ the convention $K_{\mathrm{eff}}(K)=\infty$ returns the correct zero contribution. In such settings the general definition in \textbf{Sec}.~\ref{sec:two-levels} remains a variance ratio. In particular, increasing $K$ under validation ranking need not increase $K_{\mathrm{eff}}(K)$, which is what the next subsection turns into a tradeoff.

\subsection{The Tradeoff in Selecting More Checkpoints}
\label{app:selection-bias}

At fixed $N$ and $t$, adding a lower-ranked checkpoint changes the mean anchor through $h_n(K)$ and the residual covariance through \textbf{Eq}.~\eqref{eq:effective-count-details}. A lower individual validation rank does not by itself determine the loss of the resulting average, so the relevant comparison is between the change in bias and the change in the stochastic contribution. Suppose Theorem~\ref{thm:hierarchical} applies to the same branch collection at $K$ and $K+1$, with $K$-independent $V_{\mathrm{inter}}$ and reference $V_{\mathrm{intra}}$, and write $\theta_{\mathrm{Traj\text{-}Soup}}(K)$ for the soup built from that collection at count $K$, holding $N$ and $t$ fixed. Applying \textbf{Eq}.~\eqref{eq:hierarchical} at both counts cancels $\mathcal L^\star$ and the branch term $V_{\mathrm{inter}}/N$, which leaves
\begin{equation}
\begin{aligned}
&\mathbb E[\mathcal L_Q(\theta_{\mathrm{Traj\text{-}Soup}}(K+1))]-\mathbb E[\mathcal L_Q(\theta_{\mathrm{Traj\text{-}Soup}}(K))]\\
&\qquad=B_{\mathrm{soup}}(K+1)-B_{\mathrm{soup}}(K)
+\frac{V_{\mathrm{intra}}}{N}
\left[\frac1{K_{\mathrm{eff}}(K+1)}-\frac1{K_{\mathrm{eff}}(K)}\right].
\end{aligned}
\label{eq:k-increment}
\end{equation}
When $K_{\mathrm{eff}}$ increases, the stochastic increment is negative, so an additional checkpoint improves the expected loss precisely when that reduction outweighs the bias increment. This is the mechanism behind an interior optimum, and it also leaves room for monotone improvement or monotone decline in other geometries. If $V_{\mathrm{inter},n}(K)$ varies with selection, its averaged increment must be retained, and if the branch cross-covariance varies with selection, the corresponding increment follows from \textbf{Eq}.~\eqref{eq:general-cov-loss}.

A specific bias model makes a finite optimum explicit. Take the illustrative pair $B_{\mathrm{soup}}(K)=B_{\mathrm{soup}}(1)+\gamma(K-1)^2$ with $\gamma>0$ and $K_{\mathrm{eff}}(K)=K$, which is a specialization we impose rather than a consequence of validation ranking. For continuous $K\geq1$ the resulting quadratic surrogate is
\begin{equation}
\Phi_N(K)=B_{\mathrm{soup}}(1)+\gamma(K-1)^2
+\frac{V_{\mathrm{intra}}}{NK}+\frac{V_{\mathrm{inter}}}{N}.
\label{eq:checkpoint-surrogate}
\end{equation}
Its derivatives are $\Phi_N'(K)=2\gamma(K-1)-V_{\mathrm{intra}}/(NK^2)$ and $\Phi_N''(K)=2\gamma+2V_{\mathrm{intra}}/(NK^3)>0$, so $\Phi_N$ is strictly convex on $K\geq1$. Its derivative is negative at $K=1$ and tends to $+\infty$, so there is a unique continuous minimizer $K_N^\star>1$, characterized by $2\gamma(K_N^\star)^2(K_N^\star-1)=V_{\mathrm{intra}}/N$. The map $K\mapsto K^2(K-1)$ is strictly increasing for $K>1$ because its derivative is $K(3K-2)>0$, and the right-hand side decreases in $N$, so the preferred count $K_N^\star$ decreases as branches are added. On a bounded interval of feasible counts, strict convexity means an integer optimum lies among the feasible floor and ceiling of the clipped continuous minimizer, and a sparse experimental grid requires comparing the available grid points directly. This specialization explains how additional branches can reduce the marginal value of temporal averaging, and it is consistent with the checkpoint-count shift reported in \textbf{Sec}.~\ref{sec:ablations}. Its predicted count depends on the assumed bias law, the available checkpoints, and the residual correlations, so a $K$-dependent branch component or a departure from the quadratic model changes the objective.

\subsection{The Fixed-Budget Allocation Tradeoff}
\label{app:fixed-budget}

The comparisons above keep the per-branch horizon fixed. The Limited setting instead fixes the total branch-training token budget $T$ and uses $t=T/N$, and restoring that dependence makes the competition between serial progress and averaging explicit. Let $\mathcal E_{\mathrm{Traj\text{-}Soup}}(T,N,K)$ denote the expected validation excess loss of $\theta_{\mathrm{Traj\text{-}Soup}}(N,T/N,K)$, and suppose Theorem~\ref{thm:hierarchical} holds at every compared horizon with a common reference $\theta^\star$ and curvature $H$, which is what makes those excess losses comparable. Writing the branch horizon as an explicit first argument of the bias, of the two loss-weighted variances, and of the effective count, and substituting $t=T/N$ into \textbf{Eq}.~\eqref{eq:hierarchical}, gives
\begin{equation}
\mathcal E_{\mathrm{Traj\text{-}Soup}}(T,N,K)
=B_{\mathrm{soup}}(T/N,K)+\frac{V_{\mathrm{intra}}(T/N)}{N K_{\mathrm{eff}}(T/N,K)}
+\frac{V_{\mathrm{inter}}(T/N)}N
\label{eq:fixed-budget}
\end{equation}
for any feasible $N$ and $K$. Increasing $N$ supplies more branches while shortening each one. Shorter horizons may retain greater finite-training bias, averaging may reduce both persistent and residual variation, and checkpoint availability constrains the feasible $K$. In a serial-saturation regime the cost of shortening a branch can be small enough for the averaging gains to dominate, whereas earlier in training substantial useful serial progress can favor longer branches. This is a loss-accounting identity rather than an allocation rule, since it imposes the budget constraint without supplying the horizon dependence of the individual terms. An optimal $N$ and $K$ therefore requires further information about those functions and about the feasible candidate sets, and the budget counts training tokens, whose relation to matched compute follows the experimental accounting for the compared recipes and architectures.

\subsection{Variance-Optimal Weights at Both Levels}
\label{app:equal-weights}

Theorem~\ref{thm:hierarchical} takes the two uniform averages of Trajectory Soup as given, which invites the question of whether uniform weights are the right choice or merely a convenient one. For independent errors the answer is classical, since the covariance analysis of averaging assigns equal weight to exchangeable contributions~\citep{bishop2006pattern}. Transferring that principle to our setting requires care, because the averaged quantities are parameter fluctuations weighted by the validation Hessian, and the checkpoints entering each branch anchor are selected by validation loss rather than drawn independently. SWA also averages iterates uniformly~\citep{izmailov2018averaging}, but for selected checkpoints the optimality of that choice depends on their covariance structure. This subsection makes that dependence explicit and states the result that \textbf{Sec}.~\ref{sec:equal-weights} reports in prose.

Fix $N$, $K$, $t$, and the checkpoint-selection rule, and replace the two uniform averages by deterministic nonnegative weights satisfying $\sum_{n=1}^N\alpha_n=1$ and $\sum_{j=1}^K w_{n,j}=1$ for each branch. The resulting model is $\theta_{\alpha,w}=\sum_{n=1}^N\alpha_n\sum_{j=1}^K w_{n,j}\theta_{n,(j)}$, which recovers $\theta_{\mathrm{Traj\text{-}Soup}}$ of \textbf{Eq}.~\eqref{eq:trajectory-soup} at $\alpha_n=1/N$ and $w_{n,j}=1/K$. Theorem~\ref{thm:local-loss} splits its excess loss into
\begin{equation}
\mathbb E[\mathcal L_Q(\theta_{\alpha,w})]-\mathcal L^\star=B_{\alpha,w}(K)+\mathcal V_{\alpha,w}(K),
\label{eq:weighted-loss}
\end{equation}
where $B_{\alpha,w}(K)=\frac12\|\mathbb E[\theta_{\alpha,w}]-\theta^\star\|_H^2$ and $\mathcal V_{\alpha,w}(K)=\frac12\operatorname{tr}(H\operatorname{Cov}(\theta_{\alpha,w}))$. These weights act on the covariance of the conditionally centered residuals of Appendix~\ref{app:covariance}, which we collect into the positive semidefinite matrix $[G_n^\epsilon(K)]_{jj'}=\frac12\operatorname{tr}(H\operatorname{Cov}(\epsilon_{n,j},\epsilon_{n,j'}))$. Uniform temporal weights become variance-optimal as soon as every selected checkpoint carries the same aggregate covariance with the selected set. This condition permits correlated residuals and holds for exchangeable covariance, so it is weaker than independence. Throughout, $K_{\mathrm{eff}}(K)$ keeps the definition it received for the uniformly averaged anchor in \textbf{Sec}.~\ref{sec:two-levels}.

\begin{corollary}[Variance-optimal averaging at both levels]
\label{cor:equal-weights}
Under the homogeneous contributions of Theorem~\ref{thm:hierarchical}, assume that each selected rank has the same persistent conditional shift $u_n$ and that the full collections $(u_n,\epsilon_{n,1},\ldots,\epsilon_{n,K})$ are independent across branches. If $G_n^\epsilon(K)\mathbf 1_K=\lambda_n\mathbf 1_K$ for some scalar $\lambda_n$ for every branch, then
\begin{equation}
\mathcal V_{\alpha,w}(K)\geq\frac{V_{\mathrm{intra}}}{N K_{\mathrm{eff}}(K)}+\frac{V_{\mathrm{inter}}}{N}.
\label{eq:two-level-minimum-variance}
\end{equation}
Uniform temporal and branch weights, $w_{n,j}=1/K$ and $\alpha_n=1/N$, attain this bound.
\end{corollary}

Equal row sums and positive semidefiniteness make uniform temporal weights minimize each residual quadratic form, and the branch-level minimum then follows from $\sum_{n=1}^N\alpha_n^2\geq1/N$. Appendix~\ref{app:proof-equal-weights} carries out both steps and prices the two ways of departing from Trajectory Soup separately, charging a quadratic penalty for imbalance across branches and another for unequal temporal weights within them. Because both penalties are nonnegative, the minimum in \textbf{Eq}.~\eqref{eq:two-level-minimum-variance} is exactly the stochastic contribution of \textbf{Eq}.~\eqref{eq:hierarchical}, so the two uniform averages of our method receive a single joint justification rather than two separate ones. The conclusion nevertheless constrains the variance alone, since changing the temporal weights moves the mean selected displacement and hence $B_{\alpha,w}(K)$, so uniform weights minimize the total loss only when that bias is constant across the admissible weights, as happens when all selected checkpoint means coincide. This remaining gap is what leaves room for deliberately nonuniform temporal weights, while the weighting and allocation ablations of \textbf{Sec}.~\ref{sec:ablations} settle the practical choice. Heterogeneous or correlated branches, for which inverse-variance branch weights are preferable, are treated in Appendix~\ref{app:covariance}.

\subsection{A General Criterion for Uniform Weighting}
\label{app:uniform-covariance-criterion}

The equal-row-sum hypothesis of Corollary~\ref{cor:equal-weights} is an instance of a general fact about quadratic forms on the simplex, namely that uniform weights are optimal precisely when every member carries the same aggregate covariance with the collection. Let $A\in\mathbb R^{m\times m}$ be symmetric positive semidefinite and let $p=m^{-1}\mathbf 1_m$. Then $p$ minimizes $a^\top Aa$ over $a_i\geq0$ with $\mathbf 1_m^\top a=1$ if and only if $A\mathbf 1_m=\lambda\mathbf 1_m$ for some scalar $\lambda$. For the forward direction, write $d=a-p$ so that $\mathbf 1_m^\top d=0$, and expand $a^\top Aa=p^\top Ap+2d^\top Ap+d^\top Ad$. The row-sum condition gives $Ap=(\lambda/m)\mathbf 1_m$, which annihilates the mixed term and leaves $a^\top Aa=\lambda/m+d^\top Ad$, so positive semidefiniteness establishes optimality and positivity on nonzero zero-sum directions establishes uniqueness. Conversely, if $p$ is optimal, then every entry of $p$ is positive, so $p\pm\tau d$ remains feasible for small $\tau>0$ along any zero-sum $d$. The directional derivative at $p$ must therefore vanish, and $d^\top Ap=0$ for every such $d$ forces $Ap$ to be proportional to $\mathbf 1_m$.

Applying this criterion to $A=G_n^\epsilon(K)$ gives the within-trajectory condition used in Corollary~\ref{cor:equal-weights}. For the uniform anchor, $p^\top G_n^\epsilon(K)p=V_{\mathrm{intra},n}(K)=V_{\mathrm{intra}}/K_{\mathrm{eff}}(K)$, so the common row sum must equal $KV_{\mathrm{intra}}/K_{\mathrm{eff}}(K)$. The same criterion applies to the branch covariance matrix $G(K)$ in \textbf{Eq}.~\eqref{eq:general-weighted-loss}, and in particular independent anchors with equal variance contributions have $G(K)=v(K)I_N$ and therefore admit uniform variance-optimal branch weights.

An explicit correlated temporal model shows what the criterion buys. Suppose $K\geq2$ and $G_n^\epsilon(K)=(V_{\mathrm{intra}}-c_\epsilon)I_K+c_\epsilon\mathbf 1_K\mathbf 1_K^\top$ with $0\leq c_\epsilon\leq V_{\mathrm{intra}}$, which assumes exchangeability of the loss-weighted covariance without requiring identical checkpoint means or independent residuals. Normalization gives $w_n^\top G_n^\epsilon(K)w_n=c_\epsilon+(V_{\mathrm{intra}}-c_\epsilon)\|w_n\|_2^2$, so the bound $\|w_n\|_2^2\geq1/K$ shows that uniform weights attain $V_{\mathrm{intra},n}(K)=c_\epsilon+(V_{\mathrm{intra}}-c_\epsilon)/K$, equivalently $K_{\mathrm{eff}}(K)=KV_{\mathrm{intra}}/[V_{\mathrm{intra}}+(K-1)c_\epsilon]$, and do so uniquely when $c_\epsilon<V_{\mathrm{intra}}$. At $c_\epsilon=V_{\mathrm{intra}}$ every normalized weighting carries the same residual variance, so temporal averaging provides no variance reduction at all.

\subsection{Scope of the Theoretical Claims}
\label{app:theory-scope}

The derivations require the local quadratic model of \textbf{Sec}.~\ref{sec:local-model}, finite second moments, and a shared parameterization. Theorem~\ref{thm:hierarchical} additionally assumes the shared-shift condition of Appendix~\ref{app:covariance}, independent selected anchors, and common loss-weighted contributions across branches. The general covariance treatment of Appendix~\ref{app:covariance} relaxes the last two of these, retaining the full anchor cross-covariance and allowing heterogeneous branch contributions. These assumptions are sufficient for the stated identities and are not necessary conditions for empirical merging gains.

The theory addresses population validation loss under $Q$. It does not directly imply the ordering of discrete benchmark accuracies or post-training outcomes, which are supported by the corresponding experiments. It also does not imply that all recipe deviations are zero-mean noise, that every additional branch improves the model, or that the local quadratic model remains accurate at the displacements the merges actually traverse.


\section{Proofs of the Main-Text Results}
\label{app:derivations}

This appendix proves Theorems~\ref{thm:local-loss}--\ref{thm:hierarchical} and Corollary~\ref{cor:equal-weights}. The extensions of Appendix~\ref{sec:theory-appendix} are derived where they are stated. All expectations use the fixed setup specified in \textbf{Sec}.~\ref{sec:local-model}, and every statement is proved under the local quadratic model adopted there, so all identities below are exact. Second moments are assumed finite whenever the corresponding quantities occur. The positive semidefinite matrix $H$ may be singular, so $\|\cdot\|_H$ can be a seminorm. None of the proofs requires $H$ to be invertible.

\subsection{Proof of Theorem~\ref{thm:local-loss}}
\label{app:expected-loss}

\begin{proof}
Let $\mu=\mathbb E[\vartheta]$ and write
$\vartheta-\theta^\star=(\mu-\theta^\star)+(\vartheta-\mu)$.
Expanding the quadratic form and taking expectations gives
\begin{equation}
\begin{aligned}
\mathbb E[\|\vartheta-\theta^\star\|_H^2]
={}&\|\mu-\theta^\star\|_H^2+2(\mu-\theta^\star)^\top H\mathbb E[\vartheta-\mu]\\
&+\mathbb E[(\vartheta-\mu)^\top H(\vartheta-\mu)].
\end{aligned}
\label{eq:proof-centered-expansion}
\end{equation}
The middle term vanishes because $\mathbb E[\vartheta-\mu]=0$. The last term is
\begin{equation}
\mathbb E[(\vartheta-\mu)^\top H(\vartheta-\mu)]
=\operatorname{tr}\!\left(H\mathbb E[(\vartheta-\mu)(\vartheta-\mu)^\top]\right)=\operatorname{tr}\!\left(H\operatorname{Cov}(\vartheta)\right).
\label{eq:proof-trace-identity}
\end{equation}
Taking expectations in the local quadratic model of \textbf{Sec}.~\ref{sec:local-model} now gives \textbf{Eq}.~\eqref{eq:expected-loss}.

The squared bias is nonnegative because $H\succeq0$. The variance term is nonnegative because it equals
$\mathbb E[\|H^{1/2}(\vartheta-\mu)\|_2^2]$.
The two terms therefore account for the entire expected excess loss under the local quadratic model.
\end{proof}

\subsection{Proof of Theorem~\ref{thm:branch-loss}}
\label{app:proof-branch-loss}

\begin{proof}
For fixed $n$ and $K$, decompose the centered anchor as $\bar\theta_n-\mu_n(K)=m_n(K)-\mu_n(K)+\bar\theta_n-m_n(K)$. The conditional residual has zero mean given $\mathscr F_n$, while $m_n(K)-\mu_n(K)$ is $\mathscr F_n$-measurable. Their mixed second moment is therefore zero. Expanding the covariance yields the law of total covariance,
\begin{equation}
\operatorname{Cov}(\bar\theta_n)=\mathbb E[\operatorname{Cov}(\bar\theta_n\mid\mathscr F_n)]+\operatorname{Cov}(m_n(K)).
\label{eq:proof-total-covariance}
\end{equation}
Substituting this identity into Theorem~\ref{thm:local-loss} gives \textbf{Eq}.~\eqref{eq:branch-loss}. Each covariance contribution is nonnegative after taking its trace against $H$, by the argument used in Appendix~\ref{app:expected-loss}.

If $m_n(K)=\mu_n(K)+u_n$ with the same mean-zero $u_n$ for every feasible $K$, then
$\operatorname{Cov}(m_n(K))=\operatorname{Cov}(u_n)$.
Thus $V_{\mathrm{inter},n}(K)=V_{\mathrm{inter},n}=\frac12\operatorname{tr}(H\operatorname{Cov}(u_n))$ is independent of $K$, which is the shared-shift condition assumed by Theorem~\ref{thm:hierarchical}. At $K=1$, the anchor equals $\theta_{n,(1)}$. Subtracting the $K=1$ instance of \textbf{Eq}.~\eqref{eq:branch-loss} then cancels both $\mathcal L^\star$ and $V_{\mathrm{inter},n}$, leaving
\begin{equation*}
\mathbb E[\mathcal L_Q(\bar\theta_n)]-\mathbb E[\mathcal L_Q(\theta_{n,(1)})]=B_n(K)-B_n(1)+V_{\mathrm{intra},n}(K)-V_{\mathrm{intra},n}(1),
\end{equation*}
so merging the top $K$ checkpoints of a branch improves on its best single checkpoint exactly when the intra-trajectory term falls by more than the bias rises.
\end{proof}

\subsection{Proof of Theorem~\ref{thm:hierarchical}}
\label{app:proof-hierarchical}

\begin{proof}
Linearity of conditional expectation gives $\bar\epsilon_n(K)=\bar\theta_n-m_n(K)$. Since $\mathbb E[\bar\epsilon_n(K)\mid\mathscr F_n]=0$,
\begin{equation}
\operatorname{Cov}(\bar\epsilon_n(K))=\mathbb E[\operatorname{Cov}(\bar\theta_n\mid\mathscr F_n)],
\qquad
\operatorname{Cov}(\bar\theta_n)=\operatorname{Cov}(\bar\epsilon_n(K))+\operatorname{Cov}(u_n).
\label{eq:conditional-covariance}
\end{equation}
The first identity also establishes $V_{\mathrm{intra},n}(K)=\frac12\operatorname{tr}(H\operatorname{Cov}(\bar\epsilon_n(K)))$, which is the representation used in \textbf{Sec}.~\ref{sec:two-levels}. Expanding the covariance of the finite residual average gives
\begin{equation}
\operatorname{Cov}(\bar\epsilon_n(K))
=\frac1{K^2}\sum_{j,j'=1}^K
\operatorname{Cov}(\epsilon_{n,j},\epsilon_{n,j'}).
\label{eq:intra-cov}
\end{equation}

By linearity of expectation,
$\mathbb E[\theta_{\mathrm{Traj\text{-}Soup}}]=N^{-1}\sum_{n=1}^N\mu_n(K)$,
so the term $B_{\mathrm{soup}}(K)$ defined in Theorem~\ref{thm:hierarchical} is the squared curvature-weighted bias of the soup.
\textbf{Eq}.~\eqref{eq:mean-displacement} equivalently writes this displacement as
\begin{equation}
\mathbb E[\theta_{\mathrm{Traj\text{-}Soup}}]-\theta^\star
=b_{\mathrm{com}}+\frac1N\sum_{n=1}^N\left(b_{\mathrm{rec},n}+h_n(K)\right).
\label{eq:proof-soup-mean}
\end{equation}
Independence of the selected anchors gives
\begin{equation}
\operatorname{Cov}(\theta_{\mathrm{Traj\text{-}Soup}})=\frac1{N^2}\sum_{n=1}^N\operatorname{Cov}(\bar\theta_n).
\label{eq:proof-independent-anchor-covariance}
\end{equation}
Taking the loss-weighted trace and using the common scalar contributions,
\begin{equation}
\frac12\operatorname{tr}\!\left(H\operatorname{Cov}(\theta_{\mathrm{Traj\text{-}Soup}})\right)
=\frac1{N^2}\sum_{n=1}^N\left(\frac{V_{\mathrm{intra}}}{K_{\mathrm{eff}}(K)}+V_{\mathrm{inter}}\right) =\frac{V_{\mathrm{intra}}}{N K_{\mathrm{eff}}(K)}+\frac{V_{\mathrm{inter}}}{N}.
\label{eq:proof-hierarchical-variance}
\end{equation}
Applying Theorem~\ref{thm:local-loss} to the soup gives \textbf{Eq}.~\eqref{eq:hierarchical}.

For completeness, the rows of \textbf{Tab}.~\ref{tab:averaging-methods} follow from the same identities. A rank-1 checkpoint uses $K_{\mathrm{eff}}(1)=1$. A branch anchor retains its full persistent branch contribution. A cross-trajectory merge of rank-1 checkpoints sets $K=1$, and Trajectory Soup uses the full expression. Each operation has its own mean model, so the bias terms differ across rows.
\end{proof}

\subsection{Proof of Corollary~\ref{cor:equal-weights}}
\label{app:proof-equal-weights}

\begin{proof}
Let $\mu_{n,j}=\mathbb E[\theta_{n,(j)}]$. Under the rank-independent conditional shift, $\theta_{n,(j)}=\mu_{n,j}+u_n+\epsilon_{n,j}$ with $\mathbb E[\epsilon_{n,j}\mid\mathscr F_n]=0$. Because $u_n$ is $\mathscr F_n$-measurable, $\operatorname{Cov}(u_n,\epsilon_{n,j})=0$. The centered weighted anchor is therefore $\sum_{j=1}^K w_{n,j}\theta_{n,(j)}-\sum_{j=1}^K w_{n,j}\mu_{n,j}=u_n+\sum_{j=1}^K w_{n,j}\epsilon_{n,j}$, so its loss-weighted variance equals $w_n^\top G_n^\epsilon(K)w_n+V_{\mathrm{inter}}$. Independence of the branch collections then gives
\begin{equation}
\mathcal V_{\alpha,w}(K)=\sum_{n=1}^N\alpha_n^2\left[w_n^\top G_n^\epsilon(K)w_n+V_{\mathrm{inter}}\right].
\label{eq:general-two-level-variance}
\end{equation}

Write $p=K^{-1}\mathbf 1_K$, $d_n=w_n-p$, and $v(K)=V_{\mathrm{intra}}/K_{\mathrm{eff}}(K)+V_{\mathrm{inter}}$, so that $\mathbf 1_K^\top d_n=0$. The equal-row-sum hypothesis gives $G_n^\epsilon(K)p=K^{-1}\lambda_n\mathbf 1_K$, whence $V_{\mathrm{intra},n}(K)=p^\top G_n^\epsilon(K)p=\lambda_n/K$. By the definition of $K_{\mathrm{eff}}(K)$ in \textbf{Sec}.~\ref{sec:two-levels}, this quantity equals $V_{\mathrm{intra}}/K_{\mathrm{eff}}(K)$, so the common row sum is $\lambda_n=K V_{\mathrm{intra}}/K_{\mathrm{eff}}(K)$ and $G_n^\epsilon(K)p=(V_{\mathrm{intra}}/K_{\mathrm{eff}}(K))\mathbf 1_K$. The mixed term $d_n^\top G_n^\epsilon(K)p$ therefore vanishes, and consequently
\begin{equation}
w_n^\top G_n^\epsilon(K)w_n=\frac{V_{\mathrm{intra}}}{K_{\mathrm{eff}}(K)}+d_n^\top G_n^\epsilon(K)d_n.
\label{eq:temporal-weight-penalty}
\end{equation}
Normalization of the branch weights yields
\begin{equation}
\sum_{n=1}^N\alpha_n^2=\frac1N+\sum_{n=1}^N\left(\alpha_n-\frac1N\right)^2\geq\frac1N,
\label{eq:uniform-optimal}
\end{equation}
so substituting \textbf{Eq}.~\eqref{eq:temporal-weight-penalty} into \textbf{Eq}.~\eqref{eq:general-two-level-variance} gives
\begin{equation}
\mathcal V_{\alpha,w}(K)
=\frac{v(K)}{N}+v(K)\sum_{n=1}^N\left(\alpha_n-\frac1N\right)^2+\sum_{n=1}^N\alpha_n^2 d_n^\top G_n^\epsilon(K)d_n.
\label{eq:weight-imbalance-penalty}
\end{equation}
Positive semidefiniteness of each $G_n^\epsilon(K)$ makes both penalty terms nonnegative. Uniform weights at both levels make them vanish and attain \textbf{Eq}.~\eqref{eq:two-level-minimum-variance}, and Theorem~\ref{thm:local-loss} then gives \textbf{Eq}.~\eqref{eq:weighted-loss}.

If $v(K)>0$, equality requires $\alpha_n=1/N$. If, in addition, every $G_n^\epsilon(K)$ is positive definite on the subspace orthogonal to $\mathbf 1_K$, equality also requires $w_n=K^{-1}\mathbf 1_K$ for every branch, so the joint variance minimizer is unique under these conditions. The bias term remains weight-dependent in general.
\end{proof}


\section{Compute-Scaling and Geometry Protocols}
\label{app:geometry}

This appendix describes the compute accounting and scaling fits in \textbf{Fig}.~\ref{fig:trajectory-soup-compute-allocation}, together with the geometric measurements used in \textbf{Sec}.~\ref{sec:motivation} and in the pairwise analysis of \textbf{Sec}.~\ref{sec:ablations}.

\subsection{Compute Axis and Scaling-Law Fits}
\label{app:scaling-fit}

\textbf{Compute accounting.} The horizontal axis shows estimated cumulative mid-training FLOPs on a logarithmic scale, using a sequence length of $262{,}144$ and an estimated cost of $36.74$ GFLOPs per token. Token offsets are measured from the shared fork point and summed across branches:
\begin{equation}
C \simeq 3.674\times10^{10}D,
\qquad D=\sum_{n=1}^{N}t_n,
\label{eq:flops-proxy}
\end{equation}
where $t_n$ is the token horizon of branch $n$. Equal branch horizons give $D=Nt$. This accounting covers branch training after the shared fork point and excludes the common pretraining prefix. It also leaves out checkpoint storage, the validation passes used for ranking, and the search over how many trajectories and checkpoints to combine, as discussed in Appendix~\ref{sec:limitations}. The vertical axis reports the Overall Average Accuracy of Appendix~\ref{app:evaluation-protocol} in percent, without any rescaling.

\textbf{Scaling fits.} Inspired by power-law descriptions of language-model loss~\citep{kaplan2020scaling,hoffmann2022training}, we fit an empirical asymptotic inverse power law directly to the accuracy scores:
\begin{equation}
S(C)=A-B(C/C_0)^{-\alpha},
\qquad B>0,\quad \alpha>0,
\label{eq:accuracy-power-law}
\end{equation}
where $S(C)$ is the Overall Average Accuracy in percent at cumulative compute $C$, $A$ is the fitted asymptote, and $C_0$ is the cumulative FLOPs of the first fitted point on the corresponding curve. The Raw fit uses $n=17$ Raw checkpoints from 25B to 425B tokens, and the Trajectory Soup fit uses the $n=10$ frontier points. The fitted curves are
\begin{equation}
\begin{aligned}
S_{\mathrm{raw}}(C)
&=67.2319-1.3393
\left(\frac{C}{9.1852\times10^{20}}\right)^{-0.45531},\\
S_{\mathrm{soup}}(C)
&=71.0534-2.8580
\left(\frac{C}{1.8370\times10^{21}}\right)^{-0.08818}.
\end{aligned}
\label{eq:scaling-fits}
\end{equation}
The Raw and Trajectory Soup frontier fits have $R^2=0.594$ and $R^2=0.965$, respectively. Their reference compute levels $C_0$ correspond to 25B and 50B processed tokens. These fits summarize the trends in the plotted data.

\textbf{Displayed range and annotations.} The Raw fit is shown as a solid line up to 425B tokens and as a dashed extrapolation beyond that horizon; the extrapolated segment does not enter the fit. At 100B-equivalent compute ($3.6741\times10^{21}$ FLOPs), ``Quality superiority'' compares the fitted accuracies of Trajectory Soup and Raw training, $68.36\%$ and $66.52\%$, at the same budget. ``Scaling robustness'' connects the Raw extrapolation at 600B-equivalent compute ($2.2044\times10^{22}$ FLOPs, $66.92\%$) to the observed four-trajectory result at 2400B-equivalent compute ($8.8178\times10^{22}$ FLOPs, $69.03\%$). This annotation highlights the observed improvement from expanding the trajectory budget alongside the regression of later Raw checkpoints.

\subsection{Trajectory Directions and Pairwise Alignment}
\label{app:direction-estimator}

All branches share the architecture and expert layout of Appendix~\ref{app:model-config} and are forked from a single checkpoint, so parameters correspond entrywise across branches and no permutation alignment is applied before averaging or measuring angles. For each branch, we first merge checkpoints along its training trajectory and apply PCA to the centered sequence of resulting merged checkpoints over a matched token interval. The approximately rank-one variation of these merged trajectories motivates using the unit leading principal component $d_n$ as the direction of branch $n$~\citep{zhou2026extra}, oriented along the net displacement of the merged trajectory over that interval. Because the directions are unit vectors, their pairwise alignment is the cosine $c_{nn'}=d_n^\top d_{n'}$, with associated angle $\phi_{nn'}=\arccos(c_{nn'})$. \textbf{Fig}.~\ref{fig:direction-heatmap} reports these cosine similarities for the five compatible branches of \textbf{Sec}.~\ref{sec:motivation}. Values near $1$ indicate aligned directions, while values near $0$ indicate nearly orthogonal directions.

\subsection{Interpolation Between Branch Anchors}
\label{app:interpolation-protocol}

At each matched token horizon $t$, we take the merged checkpoints $\bar\theta_n(t,K)$ and $\bar\theta_{n'}(t,K)$ from two screened branches and interpolate their parameters:
\begin{equation}
\theta_{nn'}(\lambda)
=(1-\lambda)\bar\theta_n(t,K)
+\lambda\bar\theta_{n'}(t,K),
\qquad \lambda\in[0,1].
\label{eq:interpolation}
\end{equation}
The endpoints $\lambda\in\{0,1\}$ are the within-trajectory merged anchors and $\lambda=1/2$ is their equal-weight cross-trajectory merge. We evaluate $\theta_{nn'}(\lambda)$ on the validation corpus over a grid of $\lambda$ and retain the best interpolation at each horizon. The left panel of \textbf{Fig}.~\ref{fig:interpolation} plots this quantity against consumed tokens, alongside the raw checkpoints of the two branches and their within-trajectory merges; the right panel shows the corresponding interpolation paths in a two-dimensional projection of the weight space, one path per horizon. The comparison therefore measures the additional benefit of combining branch anchors after within-trajectory averaging.

\subsection{Pairwise Cosine and Relative Interpolation Gain}
\label{app:pair-geometry}

The pairwise cosine in \textbf{Fig}.~\ref{fig:direction-gain} is computed from the same leading PCA directions as in Appendix~\ref{app:direction-estimator}. Relative gain measures the improvement of the merged model over the midpoint of the two endpoint scores. Let $A_n$ and $A_{n'}$ denote the endpoint scores and $A_{nn'}$ the score after merging the pair. Then $\Delta_{nn'}=A_{nn'}-\tfrac12(A_n+A_{n'})$, reported in points of the overall evaluation score. The subtracted term is the midpoint obtained by linearly interpolating the endpoint scores, so the gain is relative to that reference rather than to an absolute scale, and $\Delta_{nn'}>0$ means the merge improves on the mean of its endpoints. The scatter relates this gain, averaged over 17 token-budget offsets, to $c_{nn'}$ across branch pairs. Its overlaid line is an ordinary-least-squares fit with a $95\%$ confidence band for the mean response.

\textbf{Completion fields.} For the geometric protocols above, state the intra-trajectory averaging rule and window applied before PCA, the matched token interval and the number of merged points entering each fit, the variance share carried by the leading component, the interpolation grid and the token horizons of \textbf{Fig}.~\ref{fig:interpolation}, and the 17 token-budget offsets aggregated in \textbf{Fig}.~\ref{fig:direction-gain}. For the compute axis, state the $(N,t)$ configuration behind each plotted Trajectory Soup point and the number of points entering each scaling fit.


\section{Experimental Configuration}
\label{app:experiments-config}

This appendix collects the configuration of the model, the training recipe, the branch identities, and the selection and evaluation protocol. It distinguishes the fields present in the current draft from information still needed for a fully reproducible submission. The measurement conventions for the compute-scaling figure and the geometric observations are given separately in Appendix~\ref{app:geometry}. Completion fields are author-facing placeholders and are not additional experimental results.

\subsection{Default Model Configuration}
\label{app:model-config}

The core architecture is a sparse mixture-of-experts (MoE) language model with 7.9B total parameters and 1.3B parameters activated per token. The model has 24 layers, each with 128 routed experts and 8 active routed experts per token, as detailed in \textbf{Tab}.~\ref{tab:model-config}.

\begin{table}[tb]
\centering
\caption{\textbf{Default model configuration.} The architectural values are taken from the supplied model description and are not independently verified through a public model release.}
\label{tab:model-config}
\begin{tabular}{ll}
\toprule
Field & Supplied value \\
\midrule
Model name & Ling-3.0-tiny \\
Total parameters & 7.9B \\
Active parameters per token & 1.3B \\
Layers & 24 \\
Hidden width & 1{,}536 \\
First feedforward layer & Dense \\
Routed experts in each subsequent layer & 128 \\
Active routed experts per token & 8 \\
Shared experts per MoE layer & 1 \\
Expert feedforward width & 512 \\
Attention components & MLA and KDA linear attention \\
Query LoRA rank & 256 \\
KV LoRA rank & 512 \\
Vocabulary size & 157{,}184 \\
Tokens consumed before branching & 30T \\
Mid-training sequence length & 262{,}144 \\
Estimated cost per training token & 36.74 GFLOPs \\
\bottomrule
\end{tabular}
\end{table}

\subsection{Training Settings Shared by All Branches}
\label{app:training-config}

Every branch is forked from the same pretrained checkpoint and inherits the settings of \textbf{Tab}.~\ref{tab:training-config}. The fields along which the branches deliberately differ are held out of this table and listed in Appendix~\ref{app:branch-mapping}, so that the two tables together specify each recipe exactly.

\begin{table}[tb]
\centering
\caption{\textbf{Training settings shared by all branches.} Values recorded in the supplied recipe and held fixed across EXP1--EXP5. The perturbed fields are given in \textbf{Tab}.~\ref{tab:branch-mapping}.}
\label{tab:training-config}
\begin{tabular}{ll}
\toprule
Field & Value \\
\midrule
Complete branch horizon & 600B tokens \\
Optimizer & Muon \\
Warmup & 1\% of training tokens \\
Weight decay & 0.1 \\
Gradient clipping & 1.0 \\
Optimizer $\beta_1$ & 0.9 \\
Optimizer $\beta_2$ & 0.95 \\
Default checkpoint interval & 25B tokens \\
Dense-sampling ablation interval & 12.5B tokens \\
\bottomrule
\end{tabular}
\end{table}

\subsection{Branch Identity and Recipe Mapping}
\label{app:branch-mapping}

Each branch perturbs the EXP1 baseline along one of the dimensions admitted in \textbf{Sec}.~\ref{sec:problem-formulation}, except for EXP3, which varies the peak learning rate and the global batch size together. \textbf{Tab}.~\ref{tab:branch-mapping} lists the five recipes, with the perturbed fields of each branch in bold. The default comparison of \textbf{Sec}.~\ref{sec:main-results} uses EXP1--EXP3, the four- and five-branch configurations of \textbf{Fig}.~\ref{fig:trajectory-soup-compute-allocation} and \textbf{Sec}.~\ref{sec:ablations} add EXP4 and EXP5 in that order, and the geometric analyses of \textbf{Sec}.~\ref{sec:motivation} and Appendix~\ref{app:geometry} use all five. All five branches pass the compatibility screen of \textbf{Sec}.~\ref{sec:problem-formulation} at the common 600B-token horizon. The smaller-model reproduction of Appendix~\ref{app:robustness} has its own pair of branches, whose labels are local to that experiment and do not identify the same parameter checkpoints.

\begin{table}[tb]
\centering
\caption{\textbf{Branch identity and recipe mapping.} The five branches forked from the common pretrained checkpoint. Bold entries mark the fields a branch changes relative to the EXP1 baseline; all remaining settings follow \textbf{Tab}.~\ref{tab:training-config}.}
\label{tab:branch-mapping}
\resizebox{\linewidth}{!}{
\begin{tabular}{llccccc}
\toprule
Branch & Perturbed dimension
      & \begin{tabular}[c]{@{}c@{}}Data-shuffling\\seed\end{tabular}
      & \begin{tabular}[c]{@{}c@{}}Peak learning\\rate\end{tabular}
      & \begin{tabular}[c]{@{}c@{}}Global\\batch\end{tabular}
      & \begin{tabular}[c]{@{}c@{}}Schedule after\\warmup\end{tabular}
      & \begin{tabular}[c]{@{}c@{}}Muon\\momentum\end{tabular} \\
\midrule
EXP1 & Baseline                      & 1234          & $3.39\times10^{-4}$                        & 256          & Constant                                                & 0.0          \\
EXP2 & Data order                    & \textbf{1001} & $3.39\times10^{-4}$                        & 256          & Constant                                                & 0.0          \\
EXP3 & Learning rate and batch size  & 1234          & $\boldsymbol{2.40\times10^{-4}}$           & \textbf{128} & Constant                                                & 0.0          \\
EXP4 & Learning-rate schedule        & 1234          & $3.39\times10^{-4}$                        & 256          & \textbf{Decay to} $\boldsymbol{3.39\times10^{-5}}$      & 0.0          \\
EXP5 & Optimizer momentum            & 1234          & $3.39\times10^{-4}$                        & 256          & Constant                                                & \textbf{0.9} \\
\bottomrule
\end{tabular}
}
\end{table}

\subsection{Evaluation Protocol}
\label{app:evaluation-protocol}

\paragraph{Pre-training evaluation.}
We evaluate base-model checkpoints before post-training using 41 benchmark configurations across five capability categories:
\emph{(i) general knowledge and reasoning}, including ARC-Easy and ARC-Challenge~\citep{clark2018arc}, AGIEval~\citep{zhong2023agieval}, BBH and its Chinese configuration~\citep{suzgun2022bbh}, WorldSense~\citep{benchekroun2023worldsense}, PIQA~\citep{bisk2019piqa}, and HellaSwag~\citep{zellers2019hellaswag};
\emph{(ii) language understanding}, including RACE-Middle and RACE-High~\citep{lai2017race}, SQuAD~2.0~\citep{rajpurkar2018squad2}, TriviaQA~\citep{joshi2017triviaqa}, Natural Questions~\citep{kwiatkowski2019natural}, WinoGrande~\citep{sakaguchi2019winogrande}, Belebele~\citep{bandarkar2023belebele}, and CCPM~\citep{li2021ccpm};
\emph{(iii) professional knowledge and factuality}, including MMLU~\citep{hendrycks2020mmlu}, MMLU-Pro~\citep{wang2024mmlupro}, SuperGPQA~\citep{team2025supergpqa}, CMMLU~\citep{li2023cmmlu}, C-Eval~\citep{huang2023ceval}, SimpleQA~\citep{wei2024simpleqa}, and Chinese SimpleQA~\citep{he2024chinesesimpleqa};
\emph{(iv) mathematics}, including GSM8K~\citep{cobbe2021gsm8k}, GSM-Plus~\citep{li2024gsmplus}, the Chinese subset of MGSM~\citep{shi2022mgsm}, MATH and its Minerva evaluation configuration~\citep{hendrycks2021math,lewkowycz2022minerva}, CMATH~\citep{wei2023cmath}, MathBench~\citep{liu2024mathbench}, CollegeMath~\citep{tang2024mathscale}, ZKMathUnion, and GKMathUnion; and
\emph{(v) coding}, including the English and Chinese configurations of HumanEval~\citep{chen2021humaneval}, MBPP~\citep{austin2021mbpp}, HumanEval+ and MBPP+~\citep{liu2023evalplus}, LiveCodeBench~\citep{jain2024livecodebench}, CRUXEval~\citep{gu2024cruxeval}, and BIRD-SQL~\citep{li2023bird}.

\paragraph{Post-training evaluation.}
To assess whether mid-training gains persist after post-training, we evaluate models on 16 configurations grouped into six scenarios:
\emph{(i) coding}, comprising LiveCodeBench v6~\citep{jain2024livecodebench}, SciCode~\citep{tian2024scicode}, and Aider~\citep{aiderbenchmarks};
\emph{(ii) mathematics}, comprising HMMT February 2026~\citep{hmmt2026}, AIME 2025, and AIME 2026~\citep{maa2026aime};
\emph{(iii) reasoning}, comprising KOR-Bench~\citep{ma2024korbench} and GPQA~\citep{rein2023gpqa};
\emph{(iv) knowledge}, comprising MMLU-Pro~\citep{wang2024mmlupro}, Chinese SimpleQA~\citep{he2024chinesesimpleqa}, C-Eval~\citep{huang2023ceval};
\emph{(v) instruction following}, comprising, IFEval~\citep{zhou2023ifeval}, IFBench~\citep{pyatkin2025ifbench}; and
\emph{(vi) function calling}, evaluated using BFCL v4 in function-calling mode~\citep{patil2025bfcl}.

\subsection{Merging Baselines}
\label{app:protocol}

The merging scopes compared in \textbf{Sec}.~\ref{sec:setup} all act on the same candidate pools $\mathcal C_n(t)$ and all return a uniform average of some subset of those checkpoints, so they differ only in which checkpoints enter the average. We state each one in the notation of \textbf{Sec}.~\ref{sec:method}.

\textbf{Single EXP Merge.} This baseline merges within one trajectory and performs no cross-trajectory combination, so it is exactly the branch anchor $\bar\theta_n(t,K)$ of \textbf{Eq}.~\eqref{eq:topk-selection}: the uniform average of the $K$ checkpoints of $\mathcal C_n(t)$ that rank highest on the selection objective. Equivalently, it is the $N=1$ case of \textbf{Eq}.~\eqref{eq:trajectory-soup}, and it therefore exercises only the intra-trajectory term $V_{\mathrm{intra},n}(K)$ of \textbf{Eq}.~\eqref{eq:branch-loss}. Every branch produces its own merge; the tables report the strongest of them under the name Single-Trajectory Merge, evaluated at $K=12$ so that it averages as many checkpoints as the three-branch Trajectory Soup with $K=4$. Appendix~\ref{app:exp-merge} lists the merges of all branches.

\textbf{Model Soup.} This baseline is the mirror image: it merges across trajectories and performs no within-trajectory selection or averaging. It takes the final checkpoint of each branch at the evaluated horizon and averages the $N$ of them uniformly,
\begin{equation}
\theta_{\mathrm{Model\text{-}Soup}}(N,t)=\frac1N\sum_{n=1}^N\theta_n(t),
\label{eq:model-soup}
\end{equation}
in the spirit of \citet{wortsman2022model}, with membership fixed by the admitted branch set rather than chosen by a validation sweep over candidate members. It is not the $K=1$ case of \textbf{Eq}.~\eqref{eq:trajectory-soup}: that case still ranks each pool and keeps its best element, whereas \textbf{Eq}.~\eqref{eq:model-soup} fixes the position at the horizon. Contrasting the two baselines with Trajectory Soup therefore separates the benefit of combining branches from the benefit of selecting what each branch contributes, and the Limited and Extended conventions of \textbf{Sec}.~\ref{sec:setup} determine only the horizon $t$ at which the final checkpoints are taken.

\textbf{Selection objective.} \textbf{Eq}.~\eqref{eq:topk-selection} ranks candidates by validation loss $\widehat{\mathcal L}_{\mathrm{val}}$. The tables reported here instead rank them by validation accuracy, and the cardinality $K$ of every entry is the best value evaluated on that same objective. Reranking the existing checkpoints by validation loss defines a distinct selection variant whose results would have to come from an actual rerun, since the accuracy-selected tables do not establish that the two selectors agree. Because one validation signal both orders the checkpoints and chooses the reported $(N,K)$, the selection signal and the reported score are not fully independent.

\textbf{Completion fields.} Specify the validation corpus, token masking, weighting, and loss normalization behind $\widehat{\mathcal L}_{\mathrm{val}}$, together with the tasks, examples, and aggregation behind the accuracy-based selector and the benchmark weights that form the Overall Average of Appendix~\ref{app:evaluation-protocol}. Record the full $K$ grid, the tie-breaking rule, and the split held out for selection, as well as the token position of every reported optimum and the exact Limited horizon of each Model Soup entry. For the post-training comparison, state the SFT data, tokens, optimizer, schedule, and random seeds held fixed across starting checkpoints.

\subsection{Ablation Schemes}
\label{app:merge-coefficients}

Every row of \textbf{Tab}.~\ref{tab:tiny-coefficient-selection-ablation} follows the Trajectory Soup (Extended) setting of \textbf{Tab}.~\ref{tab:tiny-mid-train-main} and changes exactly one ingredient of \textbf{Eq}.~\eqref{eq:trajectory-soup}: either the coefficients applied inside a branch or the index set those coefficients are applied to. Writing $w_{n,i}\geq0$ for the weight of checkpoint $i$ in branch $n$ and $\mathcal I_n$ for its selected set, with $\sum_{i\in\mathcal I_n}w_{n,i}=1$, every evaluated variant has the form $\frac1N\sum_{n=1}^N\sum_{i\in\mathcal I_n}w_{n,i}\theta_n(i)$. The branch-level weights thus stay uniform throughout, and the ablation varies only the inner stage. The baseline row pairs the \textsc{equal} coefficients with the Top-$K$ each selection, merging $K=4$ checkpoints from each of the three branches; every other row replaces one of the two components and keeps the other at this setting.

\textbf{Coefficient schemes.} The first group fixes the selected sets to Top-$K$ each and varies the weights inside a branch. Let $j$ denote the within-trajectory rank of a selected checkpoint on the selection objective, with $j=1$ the best, and normalize each weight vector to sum to one over the $K$ checkpoints of a branch.

\textbf{\textsc{equal}.} Uniform temporal averaging, $w_j=1/K$, as prescribed by \textbf{Eq}.~\eqref{eq:trajectory-soup}. This is Trajectory Soup itself, and its overall average of 68.96 is the Trajectory Soup (Extended) entry of \textbf{Tab}.~\ref{tab:tiny-mid-train-main}. \textbf{Sec}.~\ref{sec:two-levels} identifies it as variance-optimal whenever every selected checkpoint of a branch shares the same aggregate curvature-weighted covariance with its selected set, so the remaining three schemes test how much that condition matters in practice.

\textbf{\textsc{1sqrt}.} Weights $w_j\propto\sqrt{j}$, increasing in the rank index and therefore shifting mass away from the best checkpoints toward the lower-ranked ones. The spread is mild: the $K$-th member receives $\sqrt{K}$ times the weight of the first.

\textbf{\textsc{rank}.} Weights $w_j\propto j$, increasing in the same direction but more steeply, with the $K$-th member receiving $K$ times the weight of the first. Comparing it with \textsc{1sqrt} separates the direction of the tilt from its magnitude.

\textbf{\textsc{rsqrt}.} Weights $w_j\propto1/\sqrt{j}$, decreasing in the rank index and thus concentrating mass on the highest-ranked checkpoints. Together with the two increasing schemes it brackets uniform averaging from both sides, so the group probes tilts toward better and toward worse members rather than only one of them.

\textbf{Selection strategies.} The second group fixes the coefficients to \textsc{equal} and varies which elements of the candidate pools $\mathcal C_n(t)$ enter the average.

\textbf{Top-$K$ each.} The default rule of \textbf{Eq}.~\eqref{eq:topk-selection}, taking $\mathcal I_n=\mathcal I_n(t,K)$ as the $K$ highest-ranked checkpoints of each branch. It imposes both a quality criterion and a per-branch quota, and the three alternatives below drop one of these two properties each.

\textbf{Global top-$NK$.} Keeps the quality criterion and removes the quota, taking the $NK$ highest-ranked checkpoints from the pooled candidates of all branches. The merge size matches the baseline, but the branches may be represented unequally, and a branch whose checkpoints score uniformly well can dominate the average. The comparison therefore isolates the value of balanced trajectory representation.

\textbf{Tail-$K$ per branch.} Keeps the quota and the merge size but replaces quality ranking by recency, taking the $K$ latest saved positions of each branch. Since late checkpoints are neither the best nor the most diverse, this isolates the value of ranking candidates at all, as opposed to simply averaging the end of each trajectory.

\textbf{All checkpoints.} Abandons selection entirely, setting $\mathcal I_n=\mathcal C_n(t)$ for every branch and averaging the pooled candidates uniformly. This is the Full Soup scope of \textbf{Sec}.~\ref{sec:setup}, reported here as the no-selection endpoint of the axis. It merges far more checkpoints than the baseline, which makes it the relevant control against the concern that Trajectory Soup benefits merely from averaging more models.


\section{Per-Trajectory Results of Single-Trajectory Merging}
\label{app:exp-merge}

Every branch yields its own Single EXP Merge, the within-trajectory anchor defined in Appendix~\ref{app:protocol}, and the main-text tables report only the strongest of them under the name Single-Trajectory Merge. This appendix lists all of them. The reported row is chosen separately for each table by its Overall Average, so it need not be the same branch across tables, and it is marked with $\dagger$ below. Because the main text compares Trajectory Soup against the strongest branch, this choice is conservative for the Overall Average, and comparing against any other branch would only widen the reported margins.

\subsection{Default Model at Mid-Training}
\label{app:exp-merge-mid}

\textbf{Tab}.~\ref{tab:exp-merge-mid-train} expands the Single-Trajectory Merge row of \textbf{Tab}.~\ref{tab:tiny-mid-train-main} into the three branches EXP1 (baseline), EXP2 (data shuffle), and EXP3 (learning rate and batch size) of \textbf{Tab}.~\ref{tab:branch-mapping}. Their Overall Averages span 68.34 to 68.55, with EXP1 the strongest. No branch leads in every category. EXP1 is best on general knowledge and reasoning, language modeling, and code, EXP2 on math, and EXP3 on professional knowledge. Both Trajectory Soup variants of \textbf{Tab}.~\ref{tab:tiny-mid-train-main}, at 68.72 (Limited) and 68.96 (Extended), exceed all three branches rather than only the reported one.

\begin{table}[tb]
\centering
\caption{\textbf{Single-trajectory merges of every branch on the mid-training evaluation suite.} Each row is the Single EXP Merge of one branch of \textbf{Tab}.~\ref{tab:tiny-mid-train-main}. $^\dagger$Reported as Single-Trajectory Merge in \textbf{Tab}.~\ref{tab:tiny-mid-train-main}. Bold marks the best value in each column.}
\label{tab:exp-merge-mid-train}
\resizebox{0.95\linewidth}{!}{
\begin{tabular}{lcccccc}
\toprule
Base Model & \begin{tabular}[c]{@{}c@{}}General Knowledge\\\& Reasoning\end{tabular}
      & \begin{tabular}[c]{@{}c@{}}Language\\Modeling\end{tabular}
      & \begin{tabular}[c]{@{}c@{}}Professional\\Knowledge\end{tabular}
      & Math
      & Code
      & \begin{tabular}[c]{@{}c@{}}Overall\\Average\end{tabular} \\
\midrule
EXP1 Merge$^\dagger$ & \textbf{62.80} & \textbf{85.86} & 63.70          & 72.33          & \textbf{65.36} & \textbf{68.55} \\
EXP2 Merge           & 62.79          & 85.44          & 63.97          & \textbf{72.39} & 64.99          & 68.46          \\
EXP3 Merge           & 62.69          & 85.65          & \textbf{64.02} & 71.57          & 65.34          & 68.34          \\
\bottomrule
\end{tabular}
}
\end{table}

\subsection{Default Model after Post-Training}
\label{app:exp-merge-post}

\textbf{Tab}.~\ref{tab:exp-merge-post-train} applies the SFT procedure of \textbf{Tab}.~\ref{tab:tiny-post-train-main} to each of the three merges above, and the ordering changes. EXP2 now has the highest Overall Average at 61.11, while EXP1 and EXP3 tie at 60.67. The Single-Trajectory Merge rows of \textbf{Tab}.~\ref{tab:tiny-mid-train-main} and \textbf{Tab}.~\ref{tab:tiny-post-train-main} therefore correspond to different branches, and the strongest mid-training merge is not necessarily the strongest initialization for SFT. Trajectory Soup reaches 61.39 (Limited) and 61.52 (Extended) in \textbf{Tab}.~\ref{tab:tiny-post-train-main}, again above every branch.

\begin{table}[tb]
\centering
\caption{\textbf{Single-trajectory merges of every branch after post-training.} Each row applies the SFT procedure of \textbf{Tab}.~\ref{tab:tiny-post-train-main} to the Single EXP Merge of one branch. $^\dagger$Reported as Single-Trajectory Merge in \textbf{Tab}.~\ref{tab:tiny-post-train-main}. Bold marks the best value in each column.}
\label{tab:exp-merge-post-train}
\resizebox{0.95\linewidth}{!}{
\begin{tabular}{lccccccc}
\toprule
Instruct Model & Math
      & Code
      & Knowledge
      & Reasoning
      & \begin{tabular}[c]{@{}c@{}}Instruction\\Following\end{tabular}
      & \begin{tabular}[c]{@{}c@{}}Function\\Call\end{tabular}
      & \begin{tabular}[c]{@{}c@{}}Overall\\Average\end{tabular} \\
\midrule
EXP1 Merge           & 68.22          & \textbf{47.69} & 68.98          & \textbf{64.88} & 57.54          & \textbf{51.75} & 60.67          \\
EXP2 Merge$^\dagger$ & 68.30          & 46.58          & \textbf{70.51} & 62.96          & \textbf{59.46} & 50.67          & \textbf{61.11} \\
EXP3 Merge           & \textbf{68.38} & 45.63          & 69.87          & 62.40          & 58.82          & 51.57          & 60.67          \\
\bottomrule
\end{tabular}
}
\end{table}

\subsection{Smaller MoE Model with a WSD Schedule}
\label{app:exp-merge-small}

\textbf{Tab}.~\ref{tab:exp-merge-small} lists the two branches of the 2B-parameter reproduction in \textbf{Tab}.~\ref{tab:e32a8-mid-train-main}, where EXP1 is the baseline and EXP2 changes only the data-shuffling seed. These labels are local to that experiment and do not refer to the checkpoints of the default model. The two merges are close, at 49.49 and 49.55. EXP1 is stronger on general knowledge and reasoning, language modeling, and professional knowledge, and EXP2 on math and code. Trajectory Soup reaches 49.64 (Limited) and 50.07 (Extended) and exceeds both.

\begin{table}[tb]
\centering
\caption{\textbf{Single-trajectory merges of both branches of the smaller 2B-parameter MoE model.} Each row is the Single EXP Merge of one branch of \textbf{Tab}.~\ref{tab:e32a8-mid-train-main}. $^\dagger$Reported as Single-Trajectory Merge in \textbf{Tab}.~\ref{tab:e32a8-mid-train-main}. Bold marks the best value in each column.}
\label{tab:exp-merge-small}
\resizebox{0.95\linewidth}{!}{
\begin{tabular}{lcccccc}
\toprule
Model & \begin{tabular}[c]{@{}c@{}}General Knowledge\\\& Reasoning\end{tabular}
      & \begin{tabular}[c]{@{}c@{}}Language\\Modeling\end{tabular}
      & \begin{tabular}[c]{@{}c@{}}Professional\\Knowledge\end{tabular}
      & Math
      & Code
      & \begin{tabular}[c]{@{}c@{}}Overall\\Average\end{tabular} \\
\midrule
EXP1 Merge           & \textbf{47.97} & \textbf{73.21} & \textbf{43.92} & 54.09          & 34.88          & 49.49          \\
EXP2 Merge$^\dagger$ & 47.71          & 72.85          & 43.41          & \textbf{54.21} & \textbf{35.96} & \textbf{49.55} \\
\bottomrule
\end{tabular}
}
\end{table}

\end{document}